\documentclass{article}

\usepackage{microtype}
\usepackage{graphicx}
\usepackage{subfigure}
\usepackage{booktabs} 

\usepackage{hyperref}

\usepackage{caption}

\usepackage[accepted]{icml2024}

\usepackage{amsmath}
\usepackage{amssymb}
\usepackage{mathtools}
\usepackage{amsthm}

\usepackage[capitalize,noabbrev]{cleveref}

\theoremstyle{plain}
\newtheorem{theorem}{Theorem}[section]

\theoremstyle{definition}

\theoremstyle{remark}

\newcommand{\es}{\varepsilon}
\newcommand{\hu}{\hat{u}}

\usepackage[textsize=tiny]{todonotes}

\usepackage{physics}
\usepackage{adjustbox}
\usepackage{siunitx}
\usepackage{multirow}
\newcommand{\uhat}{\hat{u}}
\DeclareMathOperator*{\argmin}{arg\,min} 

\usepackage[symbol]{footmisc}

\icmltitlerunning{TENG: Time-Evolving Natural Gradient for Solving PDEs With Deep Neural Nets Toward Machine Precision}
\begin{document}

\twocolumn[
\icmltitle{TENG: Time-Evolving Natural Gradient for Solving PDEs \\ With Deep Neural Nets Toward Machine Precision}



\icmlsetsymbol{equal}{*}

\begin{icmlauthorlist}
\icmlauthor{Zhuo Chen}{email,mitphysics,iaifi}
\icmlauthor{Jacob McCarran}{email,mitphysics,equal}
\icmlauthor{Esteban Vizcaino}{email,mitphysics,equal}
\icmlauthor{Marin Soljačić}{email,mitphysics,iaifi}
\icmlauthor{Di Luo}{email,mitphysics,iaifi,harvardphysics}
\end{icmlauthorlist}

\icmlaffiliation{email}{\texttt{\{chenzhuo,mccarran,edviz,soljacic,diluo\} @mit.edu}}
\icmlaffiliation{mitphysics}{Department of Physics, Massachusetts Institute of Technology}
\icmlaffiliation{iaifi}{NSF AI Institute for Artificial Intelligence and Fundamental Interactions}
\icmlaffiliation{harvardphysics}{Department of Physics, Harvard University}

\icmlcorrespondingauthor{Di Luo}{}

\icmlkeywords{Machine Learning, ICML, PDE, PINN, TENG, TDVP, OBTI}

\vskip 0.3in
]



\printAffiliationsAndNotice{\icmlEqualContribution} 

\begin{abstract}
Partial differential equations (PDEs) are instrumental for modeling dynamical systems in science and engineering. The advent of neural networks has initiated a significant shift in tackling these complexities
though challenges in accuracy persist, especially for initial value problems. In this paper, we introduce the \textit{Time-Evolving Natural Gradient (TENG)}, generalizing time-dependent variational principles and optimization-based time integration, leveraging natural gradient optimization to obtain high accuracy in neural-network-based PDE solutions. Our comprehensive development includes algorithms like TENG-Euler and its high-order variants, such as TENG-Heun, tailored for enhanced precision and efficiency. TENG's effectiveness is further validated through its performance, surpassing current leading methods and achieving \textit{machine precision} in step-by-step optimizations across a spectrum of PDEs, including the heat equation, Allen-Cahn equation, and Burgers' equation.
\end{abstract}

\section{Introduction}

Partial differential equations (PDEs) hold profound significance in both the theoretical and practical realms of mathematics, science, and engineering. They are essential tools for describing and understanding a multitude of phenomena that exhibit variations across different dimensions and points in time. 
The study and solution of PDEs have driven advancements in numerical analysis and computational methods, as many real-world problems modeled by PDEs are too complex for analytical solutions. The long-pursued quest for an efficient and accurate numerical PDE solver continues to be a central endeavor passionately pursued by research communities. 

\begin{figure}[t]
    \centering
    \includegraphics[width=0.95\linewidth]{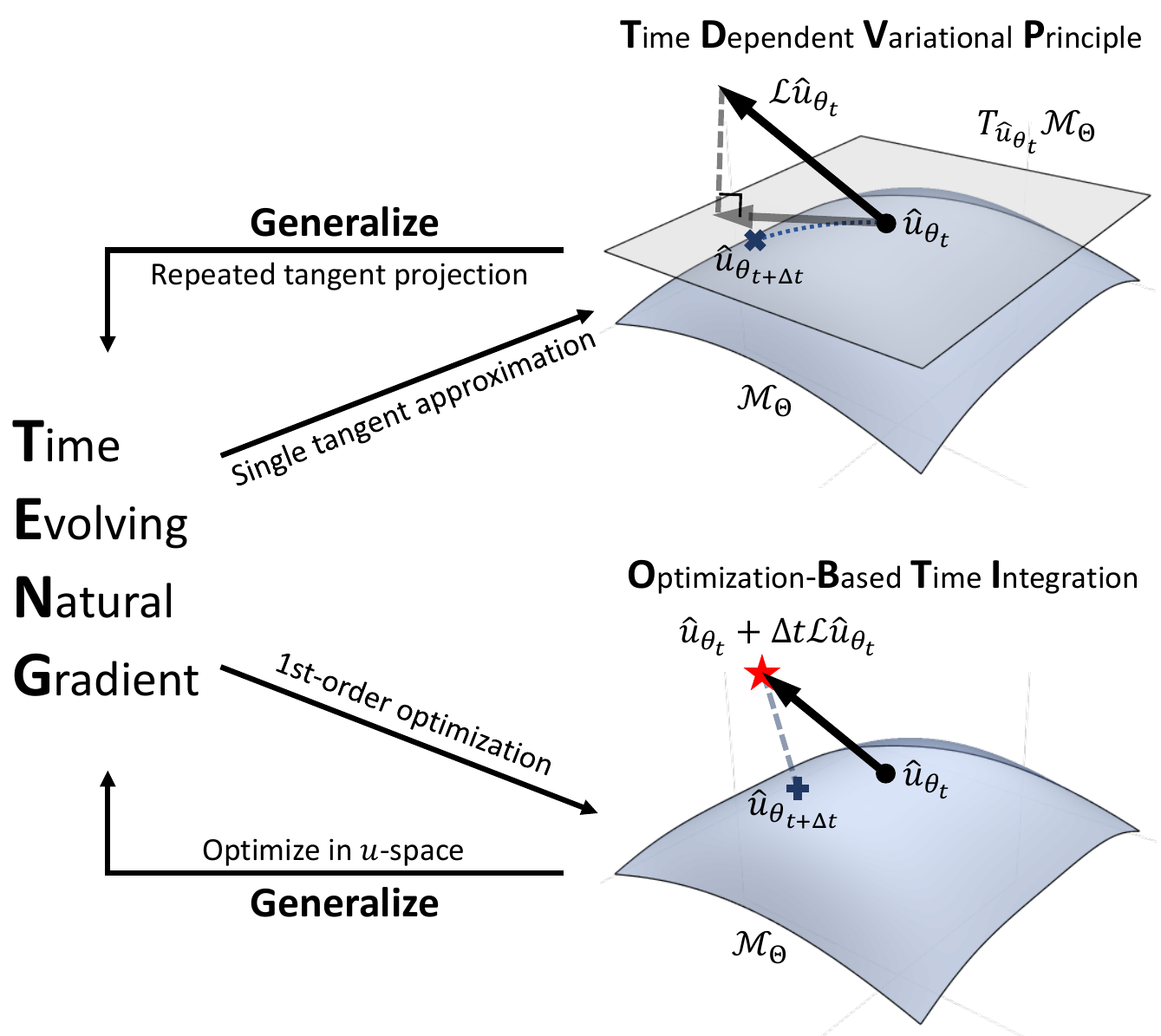}
    \caption{TENG generalizes the existing TDVP and OBTI methods. Within a single time step, TDVP projects the update direction $\mathcal{L}\uhat_{\theta_t}$ onto the tangent space of the neural network manifold $T_{\uhat_{\theta_t}}\mathcal{M}_\Theta$ at time $t$, and evolves the parameters $\theta$ according to this tangent space projection. OBTI optimizes $\theta$ to obtain an approximation to the target function $\hat{u}_{\theta_t} + \Delta t \mathcal{L} \hat{u}_{\theta_t}$ on the manifold $\mathcal{M}_\Theta$. Generalizing these two methods, TENG defines the loss function directly in the $u$-space and optimizes the loss function via repeated projections to the tangent space $T_{\uhat_{\theta_t}}\mathcal{M}_\Theta$.}
    \label{fig:illustration}
\end{figure}

In recent years, the introduction of machine learning (ML) into the study of PDEs~\cite{han2017deep,yu2018deep,long2018pde,carleo2017solving, raissi2019physics,li2020fourier,lu2019deeponet,han2018solving,sirignano2018dgm, chen2022simulating, chen_2023} has marked a transformative shift in both fields, particularly highlighted in the realms of computational mathematics and data-driven discovery. Machine learning offers new possibilities for tackling the complexities inherent in PDEs, which often pose significant challenges for traditional numerical methods due to high dimensionality, nonlinearity, or chaotic behavior. By leveraging neural networks' ability to approximate complex functions, algorithms have been developed to solve, simulate, and even discover PDEs from data, circumventing the need for explicit formulations. 

Partial differential equations with initial value problems, crucial in describing the evolution of dynamical systems, represent a fundamental class within the realm of PDEs. Despite the promising advancements made by machine learning techniques in approximating the solutions to these complex PDEs, they frequently encounter difficulties in maintaining high levels of accuracy, a challenge that becomes particularly pronounced when navigating the intricate initial conditions. This challenge largely originates from the cumulative and propagative of errors in PDE solvers over time, necessitating precise solutions at each time step for accuracy. Although various training strategies, both global-in-time training~\cite{muller2023achieving} and sequential-in-time training~\cite{chen2023implicit,berman2023randomized}, have been proposed to address this issue, it continues to stand as a critical challenge in the field. 

\textbf{Contributions.} In this paper, we introduce a highly accurate and efficient approach for tackling the above challenge by introducing \textbf{Time-Evolving Natural Gradient (TENG)}. Our key contributions are three-fold and highlighted as follows:

\begin{itemize}
    \item Propose the TENG method which generalizes two fundamental approaches in the field, time-dependent variational principle (TDVP) and optimization-based time integration (OBTI), and achieves highly accurate results by integrating natural gradient with sequential-in-time optimization.
    \item Develop efficient algorithms with sparse update for the realization of TENG, including the basic TENG-Euler and the highly accurate higher-order versions, such as TENG-Heun.
    \item Demonstrate that our approach obtains orders of magnitude better performances than state-of-the-art methods such as OBTI, TDVP with sparse updates, and PINN with energy natural gradient, and achieves machine precision accuracy during per-step optimization on a variety of PDEs, including the heat equation, Allen-Cahn equation, and Burgers' equation. 
\end{itemize}

\section{Related work}

\textbf{Machine Learning in PDEs.} Machine learning has been used to solve PDEs by using neural networks as function approximators to the solutions. In general, there are two types of strategies, global-in-time optimization and sequential-in-time optimization. Global-in-time optimization includes the physics-informed neural network (PINN)~\cite{raissi2019physics,wang2023long,wang2021learning,sirignano2018dgm,wang2021spacetime}, which optimizes the neural network representation over time and space simultaneously, or deep Ritz method~\cite{weinan2021algorithms,yu2018deep} when the variational form of the PDE exists. In contrast, sequential-in-time optimization (sometimes also called neural Galerkin method) only uses the neural network to represent the solution at a particular time step and updates the neural network representation step-by-step in time. There are different approaches to achieving such updates, including time-dependent variational principle (TDVP) ~\cite{dirac1930note,koch2007dynamical,carleo2017solving,Du_2021, berman2023randomized} and optimization-based time integration (OBTI)~\cite{chen2023implicit,kochkov2018variational,gutierrez2022real,luo2022autoregressive,luo2023gauge, Sinibaldi2023unbiasingtime}. Machine Learning has also been applied to model PDEs based on data. Such data-driven approaches include neural ODE~\cite{chen2018neural}, graph neural network methods~\cite{pfaff2020learning,sanchez2020learning}, neural Fourier operator~\cite{li2020fourier}, and DeepONet~\cite{lu2019deeponet}.

\textbf{Natural Gradient.} The concept of natural gradients, first introduced by Amari~\cite{amari1998natural} has become a cornerstone in the evolution of optimization techniques within machine learning. These methods modify the update direction in gradient-based optimization as a second-order method, typically involving using the Fisher matrix. Distinct from traditional gradient methods due to its consideration of the underlying data geometry, natural gradient descent leads to faster and more effective convergence in various scenarios. Natural gradient descent and its variants have found widespread application in areas such as neural network optimization~\cite{peters2003reinforcement,pascanu2013revisiting,zhang2019fast}, reinforcement learning~\cite{peters2003reinforcement,kakade2001natural}, quantum optimization~\cite{stokes2020quantum}, and PINN training~\cite{muller2023achieving}.

\section{Problem Formulation and Challenges} \label{sec:prelim}

\subsection{Problem formulation}
Given a spatial domain $\mathcal{X} \subseteq \mathbb{R}^d$ and temporal domain $\mathcal{T} \subset \mathbb{R}$, let $u$ be a function $\mathcal{X} \times \mathcal{T} \rightarrow \mathbb{R}$ that satisfies the following initial value problem of a PDE
\begin{equation}
\begin{aligned}
    \frac{\partial u(x, t)}{\partial t} &= \mathcal{L} u(x, t) \quad \text{for} \quad (x,t)\in\mathcal{X}\times\mathcal{T} \quad \text{and} \\
    u(x, 0) &= u_0(x),
\end{aligned}
\end{equation}
with appropriate boundary conditions. The sequential-in-time optimization approach uses neural network to parameterize the solution of the PDE at a particular time step $t\in \mathcal{T}$ as $\uhat_{\theta_t}(x): \Theta\times\mathcal{X} \rightarrow \mathbb{R}$, where the parameters have an explicitly time dependence $\theta_t: \mathcal{T} \rightarrow \mathbb{R}^{N_p}$ (with $N_p$ the number of parameters) and evolves over time. To solve the PDE, the neural network is first optimized to match the initial condition $\uhat_{\theta_0}(x) = u_0(x)$, and then optimized in a time-step-by-time-step fashion to update the parameters.

We contrasted this with the global-in-time optimization method, such as PINN~\cite{raissi2019physics}, where the neural network is used to parameterize the solution for all time $\uhat_\theta(x, t)$ with a single set of parameters. In this context, a loss function that gives rise to the global solution of the PDE is used to optimize the parameters.

\subsection{Time Dependent Variational Principle} \label{sec:tdvp}

Time-dependent variational principle (TDVP) is an existing sequential-in-time method. It aims to derive an ODE in the parameter $\theta$-space based on the function $u$-space PDE (Fig.~\ref{fig:illustration}). The most commonly used projection method is the Dirac--Frenkel variational principle~\cite{dirac1930note}, which defines the ODE by solving the following least square problem at each time step
\begin{equation}\label{eq:tdvp}
    \partial_t{\theta} = \argmin_{\partial_t \theta \in \mathbb{R}^{N_p}} \norm{\mathcal{L} \uhat_{\theta}(\cdot) - \textstyle\sum_{j}J_{(\cdot), j}\partial_t{\theta}_j}_{L^2(\mathcal{X})}^2,
\end{equation}
where $J_{(x), j}:=\partial \hat u_\theta(x)/\partial \theta_j$ is the Jacobian.

Denoting the function space of $u$ with $\mathcal{U}$, the manifold of neural network parameterized functions $\uhat_\theta$ with $\mathcal{M}_\Theta$, and the tangent space to the manifold at $\uhat_{\theta_t}$ with $T_{\uhat_{\theta_t}}\mathcal{M}_\Theta$,  Eq.~\eqref{eq:tdvp} gives the orthogonal projection of the evolution direction $\partial_t u = \mathcal{L}u$ onto the tangent space $T_{\uhat_{\theta_t}}\mathcal{M}_\Theta$ generated by the pushfoward of $\partial_t \theta$. The resulting ODE in the $\theta$-space can then be evolved in discrete time steps using numerical integrators such as the 4th-order Runge--Kutta (RK4) method. 

\textbf{\textit{Limitations.}} The Dirac--Frenkel variational principle produces the orthogonal projection of the evolution onto the tangent space $T_{\uhat_{\theta_t}}\mathcal{M}_\Theta$ at $\uhat_{\theta_t}$ during each time step. For nonzero time step sizes $\Delta t$, however, the result becomes only an approximation to the optimal projection of the target solution onto the manifold $\mathcal{M}_\Theta$. The evolution on $\mathcal{M}_\Theta$ can also deviate from the projected direction on $T_{\uhat_{\theta_t}}\mathcal{M}_\Theta$ due to nonzero time step sizes, which gives rise to the following consequence: although Eq.~\eqref{eq:tdvp} is reparameterization invariant, its nonzero $\Delta t$ version is not (see Appendix Theorem~\ref{thm:tdvp_reparam} for detail). In addition, the least square problem in Eq.~\eqref{eq:tdvp} is often ill-conditioned and the solution could be sensitive to the number of parameters, the number of samples used, and the regularization method. Although Ref.~\cite{berman2023randomized} proposed a sparse update method, where a random subset of parameters are updated at each time step, it is still hard to verify whether such choice gives the best projection in practice. Meanwhile, the solution of Eq.~\eqref{eq:tdvp} after regularization could be different from optimal.

\subsection{Optimization-based Time Integration} \label{sec:obti}

Optimization-based time integration (OBTI) is an alternative sequential-in-time method. It directly discretizes the PDE into time steps in the original function $u$-space; in each time step, OBTI first finds the next-time-step target function $u_\mathrm{target}$ based on the current-time-step $\uhat_{\theta_{t}}$, and then optimizes the next-time-step parameters $\theta_{t+\Delta t}$ by minimizing a loss function
\begin{equation} \label{eq:obti_loss}
    \theta_{t+\Delta t} = \argmin_{\theta \in \Theta} L(\uhat_{\theta}, u_\mathrm{target}).
\end{equation}
Depending on the discrete-time integration schemes used, $u_\mathrm{target}$ can be different. The most commonly used integration scheme is the forward Euler's method, where $u_\mathrm{target} = \uhat_{\theta_t} + \Delta t \mathcal{L} \uhat_{\theta_t}$
Some typical loss functions used in OBTI methods include the $L^2$-distance, the $L^1$-distance, and the KL-divergence.

\textbf{\textit{Limitations.}} Although the optimal solution to Eq.~\eqref{eq:obti_loss} gives the best approximation of $u_\mathrm{target}$ in $\mathcal{M}_\Theta$, in practice, the optimization can be very difficult with a non-convex landscape. Common optimizers such as Adam and BFGS (L-BFGS) often require a significant number of iterations to obtain an acceptable loss value. Since this optimization has to be repeated over all time steps, the accumulation of error and cost often results in poor performance. In addition, the integration scheme used in current implementations of OBTI is often limited to the forward Euler's method, which requires small time steps and further amplifies the issue of error accumulation and cost. We note that while higher-order integration schemes have been explored in prior works, they either involve applying $\mathcal{L}$ multiple times on $\uhat_\theta$ \cite{Donatella_2023} or require differentiating through $\mathcal{L} \uhat_\theta$ with respect to $\theta$ \cite{luo2022autoregressive,luo2023gauge}, both of which requires high-order differentiation, leading to stability issues and further increase of the cost.

\section{Time-Evolving Natural Gradient (TENG)}

\subsection{Generalization from TDVP and OBTI}
We first make the following observation.

{\bf Observation:} \textit{TDVP can be viewed as solving Eq.~\eqref{eq:obti_loss} with the (squared) $L^2$-distance as the loss function using a single tangent space approximation at each time step.}

\begin{proof}
At time $t$, the neural network manifold $\mathcal{M}_\Theta$ can be approximated at the point $\uhat_{\theta_t}$ by its tangent space as
\begin{equation}
    \uhat_{\theta + \delta \theta} = \uhat_{\theta} + \sum_j \frac{\partial \uhat_{\theta}}{\partial \theta_j} \delta \theta_j + \mathcal{O}(\delta\theta^2),
\end{equation}
Let $L(\uhat_{\theta + \delta \theta}, u_\mathrm{target}) = \norm{\uhat_{\theta + \delta \theta} - u_\mathrm{target}}_{L^2(\mathcal{X})}^2$. For small $\delta t$, $u_\mathrm{target} = \uhat_{\theta} + \delta t \mathcal{L}\uhat_{\theta} + \mathcal{O}(\delta t ^2)$. Keeping everything to first order, the loss function takes its minimum when
\begin{equation}
    \delta\theta = \argmin_{\delta \theta \in \mathbb{R}^{N_p}} \norm{\delta t \mathcal{L} \uhat_{\theta}(\cdot) - \sum_{j}\frac{\partial \uhat_{\theta}(\cdot)}{\partial \theta_j}{\delta\theta}_j}_{L^2(\mathcal{X})}^2.
\end{equation}
Dividing both sides by $\delta t$ recovers the TDVP (Eq.~\eqref{eq:tdvp}).
\end{proof}
 
Inspired by such observation, we introduce the time-evolving natural gradient (TENG) method, which generalizes TDVP and OBTI methods in the following way: 

\textbf{\textit{TENG solves Eq.~\eqref{eq:obti_loss} via a repeated tangent space approximation to the manifold for each time step.}}

{\bf TENG subroutine within each time step.} The key idea of TENG is shown in Fig.~\ref{fig:illustration}. During each time step, TENG minimizes the loss function in \textit{multiple iterations} (similar to OBTI), and within each iteration, it updates the parameters based on the function $u$-space gradient (of the loss function) \textit{projected to the parameter $\theta$-space} (similar to TDVP). Here, we reserve the phrase ``time step'' for physical time steps of the PDE and ``iteration'' for optimization steps within each physical time step. 

The details of TENG iterations within a single time step are shown in Subroutine~\hyperref[alg:match]{\texttt{TENG\_stepper}}, where $\alpha_n$ is the learning rate at the $n$th iteration, and the $\texttt{least\_square}(J_{(x), j}, \Delta u(x))$ should be interpreted as solving the least square problem
\begin{equation}
    \Delta\theta = \argmin_{\Delta \theta \in \mathbb{R}^{N_p}} \norm{\Delta u (\cdot) - \textstyle\sum_{j}J_{(\cdot), j}{\Delta\theta}_j}_{L^2(\mathcal{X})}^2.
\end{equation}

\begin{algorithm}[H]
\caption{{\bf Subroutine} \texttt{TENG\_stepper}}\label{alg:match}
\begin{algorithmic}
\STATE {\bfseries Input:} $\theta_\mathrm{init}$, $u_\mathrm{target}$
\STATE $n \gets 0$, $\theta \gets \theta_\mathrm{init}$
\WHILE {$n < N_\mathrm{it}$}
\STATE $\Delta u(x) \gets - \alpha_n \dfrac{\partial L(\uhat_\theta, u_\mathrm{target})}{\partial \uhat_\theta}(x)$
\STATE $J_{(x),j} \gets \dfrac{\partial \uhat_\theta(x)}{\partial \theta_j}$
\STATE $\Delta \theta \gets \texttt{least\_square}(J_{(x), j}, \Delta u(x) )$
\STATE $\theta \gets \theta + \Delta \theta$
\STATE $n \gets n+1$
\ENDWHILE
\STATE {\bfseries Output:} $\theta$
\end{algorithmic}
\end{algorithm}

We note that when Subroutine \hyperref[alg:match]{\texttt{TENG\_stepper}} is performed under certain approximations, it can be reduced to TDVP or OBTI (see Appendix~\ref{app:theory} for detail).

{\bf TENG resolves the limitations of both TDVP and OBTI.} As mentioned in Sec.~\ref{sec:tdvp}, TDVP suffers from inaccurate tangent space approximation for nonzero $\Delta t$. TENG does not suffer from this issue because of the repeated use of tangent space projections, which eventually minimizes Eq.~\eqref{eq:obti_loss} on the manifold. This also gives the following theorem as a direct consequence, which does not hold for TDVP.
\begin{theorem} 
The optimal solution of TENG is reparameterization invariant even with nonzero $\Delta t$ .
\end{theorem}
\begin{proof}
TENG achieves its optimum when $\hat{u} \in \mathcal{M}_\Theta$ is closest to $u_\mathrm{target}$ at each time step. Since a reparameterization does not change the manifold $\mathcal{M}_\Theta$ and the loss is defined in the function space, therefore, the optimal solution differs by merely a relabeling of parameters at the same point in $\mathcal{M}_\Theta$.
\end{proof}
The global convergence of natural gradient has been studied in certain non-linear neural network~\cite{zhang2019fast}. Although achieving global optimal is not theoretically guaranteed in general, in practice, the loss function is usually convex in the $u$-space, and $\hat{u}_{\theta}$ is often close to $u_\mathrm{target}$ because of small time step sizes. The result is likely to be close to global optimal. In practice, we observe the optimization can result in loss values close to machine precision ($\mathcal{O}(10^{-14})$).

In addition, while TDVP may require solving an ill-conditioned least square equation and an inaccurate solution directly affects the $\theta$-space ODE, solving the least square problem is only part of the optimization procedure for TENG, which turns out to have a smaller side effect. An inaccurate least square solution does not lead to an inaccurate solution to Eq.~\eqref{eq:obti_loss}, given sufficient iterations. 
The resulting loss value of Eq.~\eqref{eq:obti_loss} also provides a concrete metric for TENG on the accuracy during optimization.

As discussed in Sec.~\ref{sec:obti}, the main challenge for the current OBTI method lies in the difficulty of optimizing the time integrating loss function (Eq.~\eqref{eq:obti_loss}). While the loss function is a complicated non-convex function in the parameter $\theta$-space, it is usually convex in the $u$-space; therefore, it is advantageous to perform gradient descent in $u$-space and project the solution to $\theta$-space. Furthermore, TENG can also benefit from the reparametrization invariant property described above. While an efficient higher-order time integration method is still lacking in the current OBTI method, in this work we show how to incorporate higher-order methods into TENG.

{\bf TENG formulation over time steps.} The most simple time integration scheme is the forward Euler's method, which only keeps the lowest order Taylor expansion of the PDE. When integrated in the TENG method, we set $u_\mathrm{target}=\uhat_{\theta_t} + \Delta t \mathcal{L}\uhat_{\theta_t}$ and use Subroutine~\hyperref[alg:match]{\texttt{TENG\_stepper}} to solve for $\uhat_{\theta_{t+\Delta t}}$. The full algorithm is summarized in Algorithm~\hyperref[alg:euler]{\texttt{TENG\_Euler}}.

\begin{algorithm}
\caption{{\bf Algorithm} \texttt{TENG\_Euler}: A 1st-order integration scheme}\label{alg:euler}
\begin{algorithmic}
\STATE {\bfseries Input:} $\theta_{t=0}$, $\Delta t$, $T$
\STATE $t \gets 0$
\WHILE {$t < T$}
\STATE $u_\mathrm{target}(x) \gets \uhat_{\theta_{t}}(x) + \Delta t \mathcal{L}\uhat_{\theta_{t}} (x) $
\STATE $\theta_{t+\Delta t} \gets \texttt{TENG\_stepper}(\theta_{t}, u_\mathrm{target})$
\STATE $t \gets t + \Delta t$
\ENDWHILE 
\STATE {\bfseries Output:} $\theta_{t=T}$
\end{algorithmic}
\end{algorithm}

\begin{algorithm}
\caption{{\bf Algorithm} \texttt{TENG\_Heun}: A 2nd-order integration scheme}\label{alg:heun}
\begin{algorithmic}
\STATE {\bfseries Input:} $\theta_{t=0}$, $\Delta t$, $T$
\STATE $t \gets 0$
\WHILE {$t < T$}
\STATE $u_\mathrm{temp}(x) \gets \uhat_{\theta_{t}}(x) + \Delta t \mathcal{L}\uhat_{\theta_{t}} (x) $
\STATE $\theta_\mathrm{temp} \gets \texttt{TENG\_stepper}(\theta_{t}, u_\mathrm{temp})$
\STATE $u_\mathrm{target}(x) \gets \uhat_{\theta_{t}}(x) + \dfrac{\Delta t}{2} \left(\mathcal{L}\uhat_{\theta_{t}} (x) + \mathcal{L}u_{\theta_\mathrm{temp}} (x)\right)$
\STATE $\theta_{t + \Delta t} \gets \texttt{TENG\_stepper}(\theta_\mathrm{temp}, u_\mathrm{target})$
\STATE $t \gets t + \Delta t$
\ENDWHILE 
\STATE {\bfseries Output:} $\theta_{t=T}$
\end{algorithmic}
\end{algorithm}

Beyond the first-order Euler's method, Algorithm \hyperref[alg:heun]{\texttt{TENG\_Heun}} provides an example of applying second-order integration method. In this method, an intermediate target solution $u_\mathrm{temp}$ is used, and a set of intermediate parameters $\theta_\mathrm{temp}$ is trained. The intermediate parameters are used to construct $u_\mathrm{target}$ and $\theta_{t+\Delta t}$. Our method avoids terms like $\mathcal{L}^{n} \uhat_{\theta_t}$ or $\partial{\mathcal{L} \uhat_{\theta_t}}/{\partial \theta_t}$ that often appear in existing OBTI methods~\cite{Donatella_2023, luo2022autoregressive,luo2023gauge}, reducing the cost and improving numerical stability. 

\textbf{Connection to natural gradient.} We note that the algorithm outlined in Subroutine~\hyperref[alg:match]{\texttt{TENG\_stepper}} can be reformulated using the conventional Hilbert natural gradient in the form
\begin{equation}
    \Delta \theta_i = - \alpha \sum_j {G^{-1}(\theta)_{i,j}} \dfrac{\partial L(\uhat_\theta, u_\mathrm{target})}{\partial \theta_j}(x),
\end{equation}
with $G(\theta)$ the Hilbert gram matrix (see Appendix~\ref{app:theory} for detail). However, solving least square equations is more stable, with the added flexibility of choosing least square solvers. Therefore, we use the formulation in Subroutine~\hyperref[alg:match]{\texttt{TENG\_stepper}} for practical implementation.

Alternatively, Subroutine~\hyperref[alg:match]{\texttt{TENG\_stepper}} can also be viewed as a generalized Gauss--Newton method. Therefore, TENG can also be interpreted as the abbreviation of time-evolving \textit{Newton--Gauss} (also see Appendix~\ref{app:theory}).

\subsection{Complexity and Error Analysis}
The computational complexity of TENG from $t=0$ to $T$ is $\mathcal{O}(C_\mathrm{lstsq}N_\mathrm{it}T/\Delta t)$, where $C_\mathrm{lstsq} = \mathcal{O}(N_s N_p^2)$ (with $N_s$ the number of samples and $N_p$ the number of parameters) is the cost of solving the least square equation in each iteration, $N_\mathrm{it}$ is the number of iterations in each time step, and $T/\Delta t$ is the number of physical time steps. In comparison, the computational complexity of TDVP is $\mathcal{O}(C_\mathrm{lstsq}'T/\Delta t)$ and the computational complexity of OBTI is $\mathcal{O}(N_\mathrm{it}'T/\Delta t)$. Although the cost of TENG includes both $C_\mathrm{lstsq}$ and $N_\mathrm{it}$, both of the terms can be \textit{significantly smaller} than those in TDVP and OBTI due to the following reasons. 

In TDVP, the quality of the least square solution directly corresponds to the accuracy; therefore the least square equation must be solved with high accuracy and thus require high cost. Even with sparse update~\cite{berman2023randomized}, one may not be able to use too few parameters; otherwise, the update may be inaccurate. In contrast, the least square equation can be solved approximately in TENG without compromising accuracy. In this work, we design a sparse update scheme for TENG. In each iteration, we randomly sub-sample parameters and only solve the least square equation within the space of these parameters, which significantly helps reduce the cost.

OBTI, on the other hand, requires a large number of iterations in every time step to minimize the loss function, due to the difficulty of the non-convex optimization landscape. In contrast, in our TENG method, the loss values decrease to close to machine precision with only $\mathcal{O}(1)$ iterations (see Sec.~\ref{sec:results} and Appendix~\ref{app:exp} for detail). In practice, we observe that TENG is able to improve the accuracy by orders of magnitude while keeping a similar computational cost to TDVP and OBTI (also see Appendix~\ref{app:exp}).

The error of TENG is in general determined by (i) the expressivity of the neural network (ii) the optimization algorithm (iii) the time integration scheme. Based on the universal approximation theorem, with a proper choice of neural network, it is likely that the neural network is sufficiently powerful to represent the underlying solution at every time step; thus the error from (i) is small in general. Our TENG algorithm is able to achieve loss values close to machine precision at every time step; therefore the error from (ii) error is also small. Given sufficiently small errors in (i) and (ii), factor (iii) dominates the convergence property of TENG. At the same time, higher-order time-integration schemes can be integrated with TENG, in which case the error from (iii) follows the standard numerical analysis results for solving differential equations. From the above perspectives, we further contrast TENG with other algorithms. While TDVP does not have an optimization error, the projection step already introduces some errors, which can be severe for nonzero time step sizes and when the least square equation in Eq.~\eqref{eq:tdvp} is low rank~\cite{berman2023randomized}. For OBTI, the error from (ii) can be large, resulting in poor performances, in addition to the lack of efficient higher-order time-integration schemes in prior works.

TENG also permits error estimation based on the decomposition of errors above. Below, we outline the error estimation for TENG-Euler. Errors for TENG with higher-order integration methods can be estimated analogously.

Let $\es_p(t)$ be the $L^p$-error between the TENG-Euler solution and the exact solution at time $t$, $\es_p^\mathrm{EE}(t)$ between the exact solution and the solution evolved exactly according to the Euler's method, $\es_p^\mathrm{TE}(t)$ between the TENG-Euler and the solution evolved exactly according to the Euler's method, $r(\cdot, t)$ the residual function after the TENG-Euler optimization of the time step at $t$, and $\mathcal{G} := 1 + \Delta t\mathcal{L}$. 

\begin{theorem}
    The error $\es_p(t)$ is bounded by $\es_p^\mathrm{EE}(t) + \es_p^\mathrm{TE}(t)$, where $\es_p^\mathrm{EE}(t)$ is an order $\mathcal{O}(\Delta t)$ quantity, and
    $\es_p^\mathrm{TE}(t) = \norm{\sum_{n=0}^{t/\Delta t -1} \mathcal{G}^{n} r(\cdot, t-n \Delta t)}_{L^p(\mathcal{X})}$.
\end{theorem}

\begin{proof}
    See Appendix Theorem~\ref{thm:bound}.
\end{proof}

\section{Experiments}

\subsection{Equations and Setup} \label{sec:setup_equation}

\begin{figure*}[ht!]
    \centering
    \includegraphics[width=\linewidth]{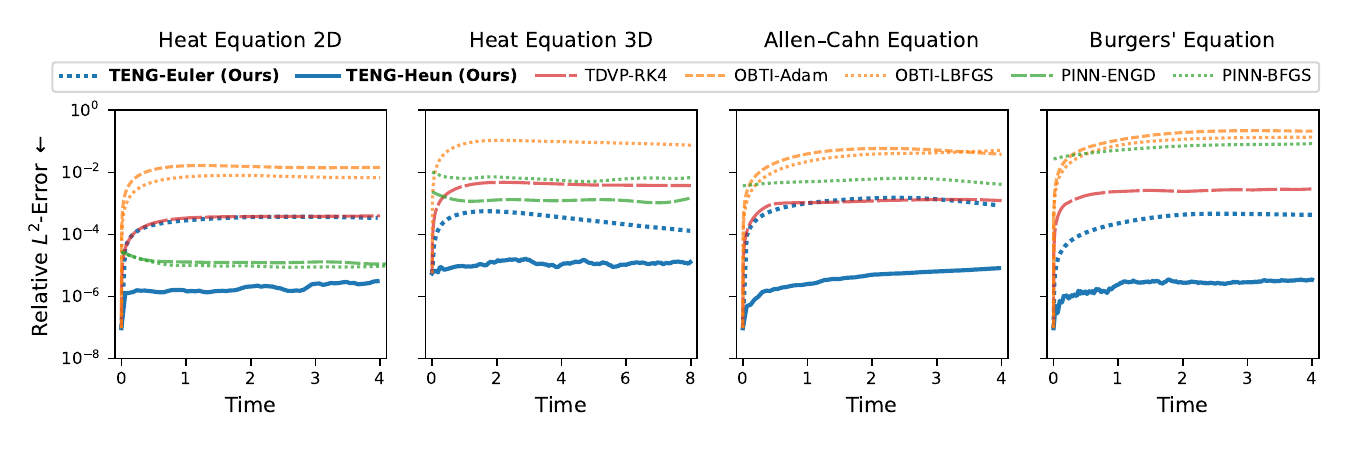}
    \caption{Benchmark of TENG, in terms of relative $L^2$-error as a function of time, against various algorithms on two- and three-dimensional heat equations, Allen--Cahn equation and Burgers' equation. All sequential-in-time methods use the same time step size $\Delta t = 0.005$ for heat and Allen--Cahn equations and $\Delta t = 0.001$ for Burgers' equation. }
    \label{fig:compare_method_all}
\end{figure*}

{\bf Heat equation.} 
The first example we choose is the two-dimensional isotropic heat equation
\begin{equation}
    \frac{\partial u}{\partial t} =  \nu \left( \frac{\partial^2 u}{\partial x_1^2} + \frac{\partial^2 u}{\partial x_2^2} \right)
\end{equation}
with a diffusivity constant $\nu = 1/10$. The heat equation describes the physical process of heat flow or particle diffusion in space.  Since it permits an analytical solution in the frequency domain, the heat equation is an ideal test bed for benchmarking (see Appendix~\ref{app:exact} for details).

{\bf Allen--Cahn equation.} 
We also consider Allen--Cahn equation 
\begin{equation}
    \frac{\partial u}{\partial t} =  \nu \left( \frac{\partial^2 u}{\partial x_1^2} + \frac{\partial^2 u}{\partial x_2^2} \right) + u - u^3
\end{equation}
with a diffusivity constant $\nu = 1/200$, which is a reaction-diffusion model that describes the process of phase separation. The Allen--Cahn equation is nonlinear and does not permit analytical solutions in general. In addition, its solution usually involves sharp boundaries and can be challenging to solve numerically. As a benchmark, we solve it using a spectral method~\cite{canuto2007spectral} (see Appendix~\ref{app:exact} for detail) and consider its solution to be a \textit{proxy} of the exact solution as the reference.

{\bf Burgers' equation.} 
We further benchmark our method on the viscous Burgers' equation
\begin{equation}
    \frac{\partial u}{\partial t} =  \nu \left( \frac{\partial^2 u}{\partial x_1^2} + \frac{\partial^2 u}{\partial x_2^2} \right) - u \left(\frac{\partial u}{\partial x_1} + \frac{\partial u}{\partial x_2}\right)
\end{equation}
with a diffusivity (viscosity) constant $\nu = 1/100$. In Appendix~\ref{app:exp}, we also explore cases with smaller $\nu$. Burgers' equation is a convection-diffusion equation that describes phenomena in various areas, such as fluid mechanics, nonlinear acoustics, gas dynamics, and traffic flow. This equation can generate sharp gradients that propagate over time, especially for small $\nu$, which can be challenging to solve. Similar to the Allen--Cahn equation, Burgers' equation does not have a general analytical solution either. Therefore, we also use the spectral method~\cite{canuto2007spectral} solution as a \textit{proxy} of the exact solution as the reference (see Appendix~\ref{app:exact} for detail).

{\bf PDE domain, boundary, and initial condition.} 
For all three equations, we first benchmark on two spatial dimensions in the domain $\mathcal{X} = [0, 2\pi) \times [0, 2\pi)$ and $\mathcal{T} = [0, 4]$, with periodic boundary condition and the following initial condition 
\begin{equation} \label{eq:initial_condition}
\begin{aligned}
    u_0(x_1, x_2) = \frac{1}{100}\bigg( &\exp \Big(3 \sin\left(x_1\right) + \sin\left(x_2\right)\Big) \\
    + &\exp \Big(-3 \sin\left(x_1\right) + \sin\left(x_2\right)\Big) \\
    - &\exp \Big(3 \sin\left(x_1\right) - \sin\left(x_2\right)\Big) \\
    - &\exp \Big(-3 \sin\left(x_1\right) - \sin\left(x_2\right)\Big)\bigg)
\end{aligned}
\end{equation}
This initial condition is anisotropic, contains peaks and valleys at four different locations, and consists of many frequencies besides the lowest frequency, which can result in challenging dynamics for various PDEs. 

For the Heat equation, we in addition consider a challenging three-dimensional benchmark, where we again choose periodic boundary conditions in the domain $\mathcal{X} = [0, 2\pi)^3$ and $\mathcal{T} = [0, 8]$. The initial condition is chosen to be a combination of sinusoidal terms in the following form so the exact solution can be analytically calculated.
\begin{equation} \label{eq:initial_condition_3d}
\begin{aligned}
    u_0(x_1, x_2, x_3) &= A_{000} \\
     & + \sum_{k_1=1}^{2}\sum_{k_2=1}^{2}\sum_{k_3=1}^{2} A_{k_1k_2k_3} \prod_{i=1}^3\cos\left(k_i x_i\right) \\
     & + \sum_{k_1=1}^{2}\sum_{k_2=1}^{2}\sum_{k_3=1}^{2} B_{k_1k_2k_3} \prod_{i=1}^3\sin\left(k_i x_i\right),
\end{aligned}
\end{equation}
where the coefficients are randomly chosen (see Appendix~\ref{app:init_cond} for coefficients used in this work). In Appendix~\ref{app:init_cond}, we also explore the heat equation on a 2D disk $\mathcal{X} = \mathcal{D}_2 = \{(x_1, x_1) : x_1^2 + x_2^2 \le 1 \}$.

For the Burgers' equation, cases with unequal domains and additional initial conditions are also explored in Appendix~\ref{app:init_cond}.

{\bf Baselines.}
While TENG is sequential-in-time, our benchmarks include both sequential-in-time (TDVP and OBTI) and global-in-time (PINN) methods. For TDVP, we choose the recently proposed sparse update method~\cite{berman2023randomized}, which has been shown to outperform previous full update methods. In addition, we use the same fourth-order Runge--Kutta integration scheme. For the OBTI method, we choose the standard Euler's integration scheme with $L(\uhat_\theta, u_\mathrm{target}) = \norm{\uhat_\theta - u_\mathrm{target}}_{L^2(\mathcal{X})}^2$ (the same loss as TENG). Both Adam and L-BFGS optimizers are used as benchmarks. For all sequential-in-time methods, we use the same time step $\Delta t = \num{5e-3}$ for the heat equation. For Allen--Cahn equation and Burgers' equation, we first compare all sequential-in-time methods with $\Delta t = \num{5e-3}$ and $\Delta t=\num{1e-3}$ respectively, before analyzing the effect of various $\Delta t$. In addition, All sequential-in-time methods share the same neural network architecture and initial parameters at $t=0$. For PINN, we test both BFGS and the recently proposed ENGD optimizer. Since Ref.~\cite{muller2023achieving} did not provide the implementation for Allen--Cahn equation and Burgers' equation, we omit the benchmark of ENGD optimizer for the two equations. We use a network architecture similar to Ref.~\cite{muller2023achieving} (see Appendix~\ref{app:hyperparam} for detail).

{\bf Error metric.}
We consider the following two error metrics:
\begin{enumerate}
    \item relative $L^2$-error at each time step 
    \begin{equation} \label{eq:rel_err}
        \varepsilon(t) = \frac{\norm{\uhat(\cdot, t)-u_\mathrm{reference}(\cdot, t)}_{L^2(\mathcal{X})}}{\norm{u_\mathrm{reference}(\cdot,  t)}_{L^2(\mathcal{X})}},
    \end{equation}
    \item global relative $L^2$-error integrated over all time steps
    \begin{equation} \label{eq:g_rel_err}
    \varepsilon_g = \frac{\norm{\uhat(\cdot, \cdot)-u_\mathrm{reference}(\cdot, \cdot)}_{L^2(\mathcal{X}\times\mathcal{T})}}{\norm{u_\mathrm{reference}(\cdot,  \cdot)}_{L^2(\mathcal{X}\times\mathcal{T})}},
    \end{equation}
\end{enumerate}
where $u_\mathrm{reference}$ refers to the analytical solution for the heat equation, and the spectral method solution for Allen--Cahn and Burgers' equation (see Appendix~\ref{app:exact} for detail).

\subsection{Results} \label{sec:results}

{\bf Benchmark against other methods.} We start the benchmark of our method against other methods described in Sec.~\ref{sec:setup_equation} in terms of both the relative $L^2$-error (Eq.~\eqref{eq:rel_err}) as a function of time (Fig.~\ref{fig:compare_method_all}) and the global relative $L^2$-error (Eq.~\eqref{eq:g_rel_err}) integrated over all time (Table~\ref{table:benchmark}). 

\begin{figure}[h!]
    \centering
    \includegraphics[width=\linewidth]{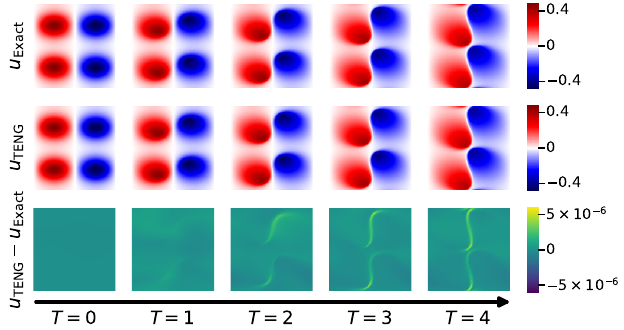}
    \caption{Reference solution, TENG solution, and the difference between them for Burgers' equation. The reference solution is generated using the spectral method, and the TENG solution shown here uses the TENG-Heun method with $\Delta t = 0.001$.}
    \label{fig:burgers_color}
\end{figure}

\begin{table*}[h!]
\centering
\adjustbox{max width=\linewidth}{
\setlength{\tabcolsep}{4pt} 
\begin{tabular}{l | c  c  c  c } 
\hline\hline& \\[-2.4ex]
\multirow{2}{*}{Method} & \multicolumn{4}{c}{Global Relative $L^2$-Error $\downarrow$} \\ \cline{2-5} & \\[-2.4ex]
& Heat (2D) & Heat (3D) & Allen--Cahn & Burgers'\\ 
\hline & \\[-2.2ex]
{\bf TENG-Euler (Ours)} & $3.006\times10^{-4}$ & $\it{3.664\times10^{-4}}$  & $\it{1.249\times10^{-3}}$ & $\it{3.598\times10^{-4}}$\\
{\bf TENG-Heun (Ours)} & $\boldsymbol{1.588\times10^{-6}}$ & $\boldsymbol{1.139\times10^{-5}}$ & $\boldsymbol{6.187\times10^{-6}}$ & $\boldsymbol{2.643\times10^{-6}}$\\
TDVP-RK4              & $3.279\times10^{-4}$              & $3.841\times10^{-3}$              & $1.258\times10^{-3}$         & $2.437\times 10^{-3}$  \\
OBTI-Adam              & $1.391\times10^{-2}$              &      --         & $4.966\times10^{-2}$              & $1.696\times 10^{-1}$  \\
OBTI-LBFGS              & $6.586\times10^{-3}$              & $8.743\times10^{-2}$              & $4.180\times10^{-2}$              & $1.047\times 10^{-1}$  \\
PINN-ENGD            & $1.403\times10^{-5}$         &    $2.846\times10^{-3}$      & --              & --  \\
PINN-BFGS            & $\it{1.150\times10^{-5}}$         & $1.389\times10^{-2}$         & $5.540\times10^{-3}$              & $6.538\times 10^{-2}$  \\
\hline
\end{tabular}
}
\caption{Benchmark of TENG, in terms of global relative $L^2$-error, against various algorithms on the heat equation, Allen--Cahn equation and Burgers' equation. The best result in each column is marked in boldface and the second best result is marked in italic font. Here, the same $\Delta t$ as in Fig.~\ref{fig:compare_method_all} is used.}
\label{table:benchmark}
\end{table*}

In all cases, our TENG-Heun method achieves results orders of magnitude better compared to other methods. Upon closer inspection, our TENG method with Euler's integration scheme is already comparable to or better than the TDVP method with the RK4 integration scheme. In addition, TENG-Euler is significantly better than OBTI with both Adam and L-BFGS optimizers, both of which use the same integration scheme. In Fig.~\ref{fig:burgers_color}, We show the difference between TENG-Heun and the reference solution by plotting the function evolution over time. It can be seen that our method traces closely with the reference solution with a tiny deviation on the order of $\mathcal{O}(10^{-6})$. In Appendix~\ref{app:exp}, we show additional details of runtime, and the relation between runtime and performance.

\textbf{Convergence speed and machine precision accuracy.} We further demonstrate the convergence of TENG-Euler in Fig.~\ref{fig:allen_cahn_loss} compared to OBTI-Adam and OBTI-LBFGS. In a single time step, TENG achieves a training loss value (with the squared $L^2$ distance as the loss function) close to machine precision $\mathcal{O}(10^{-14})$ with only a few iterations, while OBTI can only get to $\mathcal{O}(10^{-7})$ loss after a few hundred iterations. We also plot the final loss of each time step optimization and show that TENG stably reaches the machine precision for all time, which is seven orders of magnitude better than OBTI. Our results have shown the high accuracy of TENG compared to the existing approaches (see Appendix~\ref{app:exp} for additional results). Since the final loss values are near machine precision for all time steps, we believe the main error source of TENG-Euler comes from the Euler's expansion, instead of the neural network approximation. This is further verified later during time integration scheme comparisons.

\begin{figure}[ht!]
    \centering
    \includegraphics[width=\linewidth]{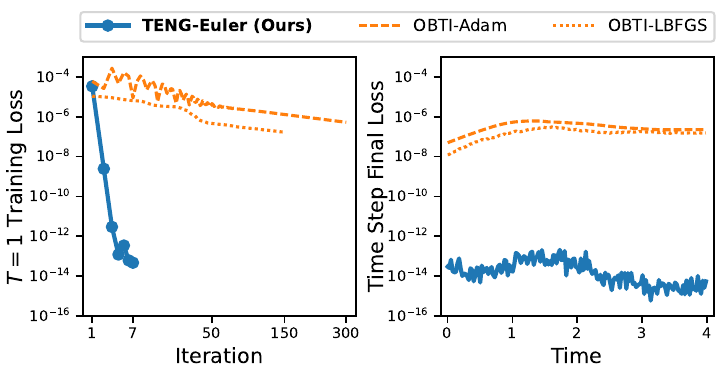}
    \caption{Training loss during the time step at $T=1$ and final training losses for all time steps for the TENG-Euler method and the two OBTI methods for Allen--Cahn equation.}
    \label{fig:allen_cahn_loss}
\end{figure}

{\bf Compare time integration schemes.} We further examine the effects of time integration schemes on TENG and compare our TENG-Euler and TENG-Heun methods with different time step sizes.

\begin{figure}[ht!]
    \centering
    \includegraphics[width=\linewidth]{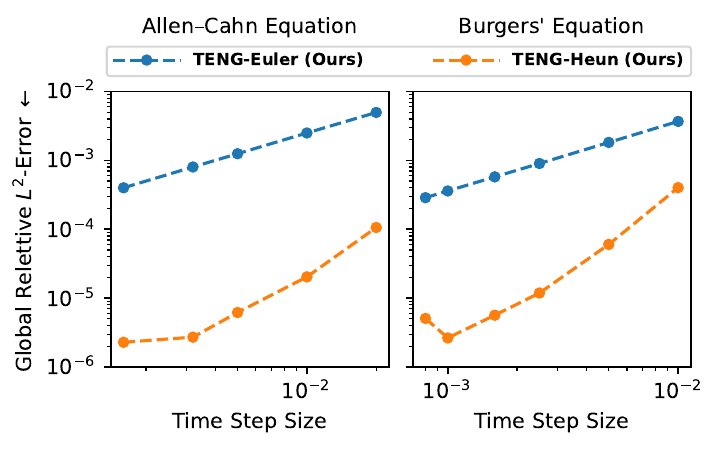}
    \caption{Comparison of different time integration schemes of TENG with respect to the time step sizes on Allen--Cahn equation and Burgers' equation, using global relative $L^2$-error as a metric.}
    \label{fig:compare_integrator_dt}
\end{figure}

In Fig.~\ref{fig:compare_integrator_dt}, we show the global relative $L^2$-error (defined in Eq.~\eqref{eq:g_rel_err}) as a function of time step size $\Delta t$. It can be seen from the figure that while TENG-Euler already achieves a global relative $L^2$-error of $\mathcal{O}(10^{-4})$ for small $\Delta t$, using a higher order integration scheme significantly reduces the error to $\mathcal{O}(10^{-6})$. In addition, TENG-Heun can maintain the low error even at relatively large $\Delta t$, signifying the advantage of our implementation of higher-order integration schemes. We note that, for small time step sizes, the accumulation of per-step error dominates, while for large time step sizes, the discretization error from the integration scheme dominates, resulting in the TENG-Heun with the smallest $\Delta t$ not as good as larger $\Delta t$. In addition, the curves for TENG-Euler and TENG-Heun have different slopes. Both phenomena are consistent with numerical analysis results for traditional PDE solvers. Additional explorations with TENG-RK4 method can be found in Appendix~\ref{app:exp}.

\section{Discussion and Conclusion}
We introduce \textit{Time-Evolving Natural Gradient}, a novel approach that generalizes time-dependent variational principles and optimization-based time integration, resulting in a highly accurate and efficient PDE solver utilizing natural gradient. TENG, encompassing algorithms like TENG-Euler and advanced variants such as TENG-Heun, significantly outperforms existing state-of-the-art methods in accuracy, achieving machine precision in solving a range of PDEs. For future work, it would be interesting to explore the application of TENG in more diverse and complex real-world scenarios, particularly in areas where traditional PDE solutions are currently unfeasible. While this work is focused on two- and three-dimensional (spatial) scalar-valued PDEs with periodic boundary conditions, the same method can be considered for generalizing to vector-valued PDE in other numbers of dimensions, and other boundary conditions, such as the Dirichlet boundary condition or the Neumann boundary condition. 
It will also be important to develop TENG for broader classes of PDEs besides initial value problems with applications to nonlinear and multi-scale physics PDEs in various domains. Advancing TENG's integration with cutting-edge machine learning architectures and optimizing its performance for large-scale computational tasks will be a vital area of research for computational science and engineering.

\section*{Acknowledgements}

The authors acknowledge support from the National Science Foundation under Cooperative Agreement PHY-2019786 (The NSF AI Institute for Artificial Intelligence and Fundamental Interactions, \url{http://iaifi.org/}). This material is based upon work supported by the U.S. Department of Energy, Office of Science, National Quantum Information Science Research Centers, Co-design Center for Quantum Advantage (C2QA) under contract number DE-SC0012704. This material is also in part based upon work supported by the Air Force Office of Scientific Research under the award number FA9550-21-1-0317. The research was sponsored by the United States Air Force Research Laboratory and the Department of the Air Force Artificial Intelligence Accelerator and was accomplished under Cooperative Agreement Number FA8750-19-2-1000. The authors acknowledge the MIT SuperCloud and Lincoln Laboratory Supercomputing Center for providing (HPC, database, consultation) resources that have contributed to the research results reported within this paper. Some computations in this paper were run on the FASRC cluster supported by the FAS Division of Science Research Computing Group at Harvard University.

\section*{Impact Statement}
Through the advancement of the \textit{Time-Evolving Natural Gradient} (TENG), which solves partial differential equations (PDEs) with enhanced accuracy and efficiency, our work exhibits broader impact spans multiple disciplines, including but not limited to, climate modeling, fluid dynamics, materials science, and biomedical engineering. While the primary goal of this work is to push forward computational techniques within the field of machine learning, it inherently carries the potential for significant societal benefits, such as improved environmental forecasting models, more efficient engineering designs, and advancements in medical technology. Ethically, the deployment of TENG should be approached with consideration to ensure that enhanced computational capabilities translate into positive outcomes without unintended consequences given its accuracy and reliability in real-world applications. There are no specific ethical concerns that we feel must be highlighted at this stage; however, we acknowledge the importance of ongoing evaluation of the societal and ethical implications as this technology is applied. This acknowledgment aligns with our commitment to responsible research and innovation, understanding that the true value of advancements in machine learning is realized through their contribution to societal progress and well-being.

\nocite{langley00}

\newpage

\bibliography{reference}
\bibliographystyle{icml2024}

\newpage
\appendix
\onecolumn
\renewcommand\theequation{\thesection.\arabic{equation}}
\setcounter{equation}{0}
\renewcommand\thefigure{\thesection.\arabic{figure}}
\setcounter{figure}{0}
\renewcommand\thetable{\thesection.\arabic{table}}
\setcounter{table}{0}

\section{Additional Theoretical Results} \label{app:theory}

\begin{theorem} \label{thm:tdvp_reparam}
    TDVP is not reparameterization invariant with $\Delta t$.
\end{theorem}
\begin{proof}
    We will construct an explicit counter-example. For simplicity, consider a zero-dimensional PDE (an ODE) $\partial_t u = u$, whose solution is $u = u_0  \exp(t)$. Let $\hat{u}_\theta = \theta$ and $\hat{v}_\phi = \exp \phi$, which are just reparameterizations of each other. In the parameter space, the two ODEs read $\partial_t \theta = \theta$ and $\partial_t\phi = 1$. Let both of them evolve for a discrete time step $\Delta t$ from $t=0$, we have $\theta_{\Delta t} = \theta_0 + \Delta t \theta_0$ and $\phi_{\Delta t} = \phi_0 + \Delta t$. Plugging back into the functions, $\uhat_{\theta_{\Delta t}} = \uhat_{\theta_{0}} + \Delta t \uhat_{\theta_{0}}$ and $\hat{v} = \hat{v}_{\theta_{0}} \exp(\Delta t)$ It is evidential that the two parameterizations give different solutions.
\end{proof}

{\bf TENG can be reduced to TDVP under certain assumptions.} For simplicity, we will be focusing on the first-order Euler's method. Consider Subroutine~\hyperref[alg:match]{\texttt{TENG\_stepper}}. Let the loss function $L(\uhat_{\theta_t}, u_\mathrm{target}) = \norm{\uhat_{\theta_t}- u_\mathrm{target}}_{L^2(\mathcal{X})}^2$ and $N_\mathrm{it} = 1$. For simplicity, let $u_\mathrm{target} = \uhat_{\theta_t} + \Delta t \mathcal{L} \uhat_{\theta_t}$ be the first-order Euler expansion. Then,
\begin{equation}
    \frac{\partial L}{\partial \uhat_{\theta_t}}(x) = 2 (\uhat_{\theta_t}(x) - u_\mathrm{target}(x)) \equiv 2 \Delta t \mathcal{L} \uhat_{\theta_t}.
\end{equation}
Choosing $\alpha = 1/2$, we have $\Delta u = \Delta t \mathcal{L} \uhat_{\theta_t}$. Then, the least square equation becomes
\begin{equation}
     \Delta\theta = \argmin_{\Delta \theta \in \mathbb{R}^{N_p}} \norm{\Delta t \mathcal{L} \uhat_{\theta_t} (\cdot) - \textstyle\sum_{j}J_{(\cdot), j}{\Delta\theta}_j}_{L^2(\mathcal{X})}^2,
\end{equation}
which is the same as the TDVP algorithm with nonzero time step sizes.

{\bf TENG can be reduced to OBTI under certain assumptions.} Let $N_\mathrm{it} > 0$. As mentioned in the main paper, approximate methods can be used to solve the least square equation in Subroutine~\hyperref[alg:match]{\texttt{TENG\_stepper}}. Here, let its solution be approximated by a single gradient descent, which gives rise to
\begin{equation}
    \Delta \theta_j = \int J_{(x), j} \Delta u(x) \mathrm{d} x \equiv -\alpha \int \frac{\partial{\uhat_\theta(x)}}{\partial \theta_j} \frac{\partial L}{\partial \uhat_\theta}(x) \mathrm{d} x = -\alpha \frac{\partial L}{ \partial \theta_j},
\end{equation}
which reduces to the regular gradient descent in the $\theta$-space with many iterations.

{\bf Hilbert natural gradient formulation of Subroutine~\hyperref[alg:match]{\texttt{TENG\_stepper}}}. Consider the least square equation
\begin{equation} \label{eq:app_teng_lstsq}
    \Delta\theta = \argmin_{\Delta \theta \in \mathbb{R}^{N_p}} \norm{\Delta u (\cdot) - \textstyle\sum_{j}J_{(\cdot), j}{\Delta\theta}_j}_{L^2(\mathcal{X})}^2.
\end{equation}
It's solution is given by the normal equation $J^T J \Delta\theta = J^T \Delta u$
where we use the matrix notation and omit the indices. The solution to the normal equation is given by 
$\Delta\theta = (J^T J)^{-1} J^T \Delta u.$
Notice that 
\begin{equation}
    (J^T \Delta u)_j = \int J_{(x), j} \Delta u(x) \mathrm{d} x \equiv -\alpha \int \frac{\partial{\uhat_\theta(x)}}{\partial \theta_j} \frac{\partial L}{\partial \uhat_\theta}(x) \mathrm{d} x = -\alpha \frac{\partial L}{ \partial \theta_j}.
\end{equation}
In addition,
\begin{equation}
    (J^T J)_{i,j} = \int \frac{\partial{\uhat_\theta(x)}}{\partial \theta_i}  \frac{\partial{\uhat_\theta(x)}}{\partial \theta_j} \mathrm{d} x \equiv G_{i, j}(\theta),
\end{equation}
where $G(\theta)$ is the Hilbert gram matrix~\cite{muller2023achieving}. Therefore, Eq.~\eqref{eq:app_teng_lstsq} can be written as the Hilbert natural gradient descent
\begin{equation} \label{eq:app_teng_ngd}
    \Delta \theta_i = - \alpha \sum_j {G^{-1}(\theta)_{i,j}} \dfrac{\partial L(\uhat_\theta, u_\mathrm{target})}{\partial \theta_j}(x).
\end{equation}
We note that while these two formulations are mathematically equivalent, the least square formulation has a few practical advantages. First, it allows for more stable numerical solvers. In general, the Hilbert gram matrix has a condition number twice as large as the original Jacobian matrix. If the original least square equation is ill-conditioned, Eq.~\eqref{eq:app_teng_ngd} is even worse. In addition, when the least square equation is underdetermined, solving the original least square problem gives the minimum norm solution, whereas Eq.~\eqref{eq:app_teng_ngd} has to be solved with pseudo-inverse, which can be numerically unstable in practice.

{\bf Generalized Gauss--Newton formulation of Subroutine~\hyperref[alg:match]{\texttt{TENG\_stepper}}}. Let the loss function be the squared $L^2$-distance. Define $r(x):= \uhat_\theta(x) - u_\mathrm{target}(x)$.The derivative of loss in function space is given by
\begin{equation}
    \frac{\partial L}{\partial \uhat_\theta}(x) = 2 (\uhat_\theta(x) - u_\mathrm{target}(x)) \equiv 2 r(x).
\end{equation}
In matrix notation, the iteration above becomes
\begin{equation}
    \Delta \theta = -\alpha (J^T J)^{-1} J^T r.
\end{equation}
When $\alpha = 1/2$, this reduces to one iteration of the Gauss--Newton method. Therefore, Subroutine~\hyperref[alg:match]{\texttt{TENG\_stepper}} can also be viewed as a generalized Gauss--Newton method.

\begin{theorem}\label{thm:bound}
    The error $\es_p(t)$ is bounded by $\es_p^\mathrm{EE}(t) + \es_p^\mathrm{TE}(t)$, where $\es_p^\mathrm{EE}(t)$ is an order $\mathcal{O}(\Delta t)$ quantity, and
    $\es_p^\mathrm{TE}(t) = \norm{\sum_{n=0}^{t/\Delta t -1} \mathcal{G}^{n} r(\cdot, t-n \Delta t)}_{L^p(\mathcal{X})}$.
\end{theorem}

\begin{proof}
Denote $D(\cdot, t) = u(\cdot, t) - \hu_{\theta}(\cdot, t) = (u(\cdot,t) - u^\mathrm{Eu}(\cdot,t)) + (u^\mathrm{Eu}(\cdot,t) - \hu_{\theta}(\cdot,t)) \equiv D^{EE}(\cdot,t) + D^{TE}(\cdot,t)$, where $u(\cdot,t)$ is the exact solution, $u^\mathrm{Eu}(\cdot,t)$ is the solution from Euler method, $\hu_{\theta}(\cdot,t)$ is the TENG-Euler solution at time $t$. By definition, $\es_p(t) = \norm{D(\cdot, t)}_{L^p(\mathcal{X})}$, $\norm{u(\cdot,t) - u^\mathrm{Eu}(\cdot,t)}_{L^p(\mathcal{X})} = \norm{D^{EE}(\cdot, t)}_{L^p(\mathcal{X})} =\es^{EE}(t)$, and $\norm{u^\mathrm{Eu}(\cdot,t) - \hu(\cdot,t)}_{L^p(\mathcal{X})} = \norm{D^{TE}(\cdot, t)}_{L^p(\mathcal{X})} = \es^{TE}(t)$. It follows by the triangular inequality that $\es_p(t) \leq \es_p^\mathrm{EE}(t) + \es_p^\mathrm{TE}(t)$. Since the Euler method is a first-order method, $\es_p^\mathrm{EE}(t)$ has an error of order $O(\Delta t)$. 

Denote the optimization error in time $t$ as $r(\cdot,t)$, such that $\hu_{\theta}(\cdot,t+\Delta t) - \mathcal{G} \hu_{\theta}(\cdot,t)=r(\cdot,t)$. It follows that $u^\mathrm{Eu}(\cdot,t+\Delta t)-D^{TE}(\cdot,t+\Delta t) - \mathcal{G}(u^\mathrm{Eu}(\cdot,t)-D^{TE}(\cdot,t))=r(\cdot,t)$, which implies that $D^{TE}(\cdot,t+\Delta t) = \mathcal{G} D^{TE}(\cdot,t)-r(\cdot,t)$ due to the cancellation of $u^{Eu}(\cdot,t)-\mathcal{G} u^{Eu}(\cdot,t)$ from the exact Euler method. By induction, $\es_p^\mathrm{TE}(t) = \norm{\sum_{n=0}^{t/\Delta t -1} \mathcal{G}^{n} r(\cdot, t-n \Delta t)}_{L^p(\mathcal{X})}$.

\end{proof}

\section{Additional Experimental Results} \label{app:exp}
In this section, we show additional benchmark results. In Fig.~\ref{fig:allen_cahn_loss_many}, we show the training losses of TENG and OBTI methods during the time steps at $T=0.8, 1.6, 2.4, 3.2, 4.0$ for Allen--Cahn equation. For Euler's integration scheme, we compare our TENG-Euler method with both OBTI-Adam and OBTI-LBFGS algorithms and find that our algorithm consistently achieves loss values orders of orders of magnitudes better than OBTI, with only 7 iterations. For TENG-Heun method, each time step requires training a set of intermediate parameters. Therefore, each figure includes two curves.  As shown in the figure, all stages converge to machine precision within a small number of iterations. 

\begin{figure}[ht!]
    \centering
    \includegraphics[width=\linewidth]{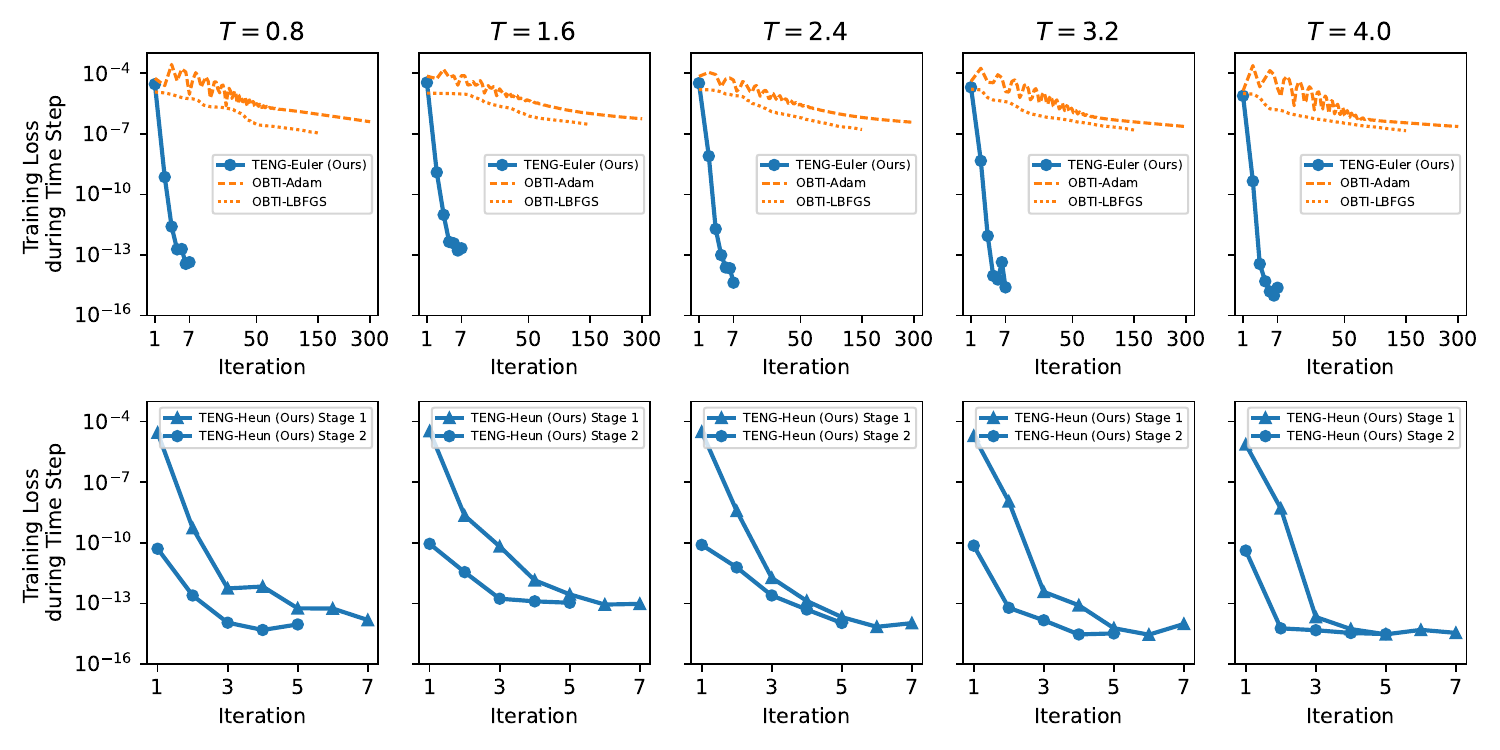}
    \caption{Training loss during many time steps for TENG-Euler, TENG-Heun, and the two OBTI methods for Allen--Cahn equation with a time step size $\Delta t = 0.005$. TENG-Heun method requires two training stages, one for $\theta_\mathrm{temp}$, and the other for $\theta_{t+\Delta t}$. Therefore, each figure contains two curves.}
    \label{fig:allen_cahn_loss_many}
\end{figure}

In Fig.~\ref{fig:heat_color_all} \ref{fig:burgers_color_all} and \ref{fig:allen_cahn_color_all}, we show the density plots for the two-dimensional heat equation, Burgers' equation and Allen--Cahn equation. In each figure, we plot the reference solution (see Appendix~\ref{app:exact} for details on obtaining the reference solution), the TENG-Heun solution, the TDVP-RK4 solution, the OBTI-LBFGS solution and the PINN-BFGS solution, and their difference to the reference solution. In all cases, the TENG-Heun solution closely tracks the reference solution, with a maximum error of order $\mathcal{O}(10^{-6})$, whereas solutions generated by other methods can have relatively larger solutions.

\begin{figure}[ht!]
    \centering
    \includegraphics[width=0.9\linewidth]{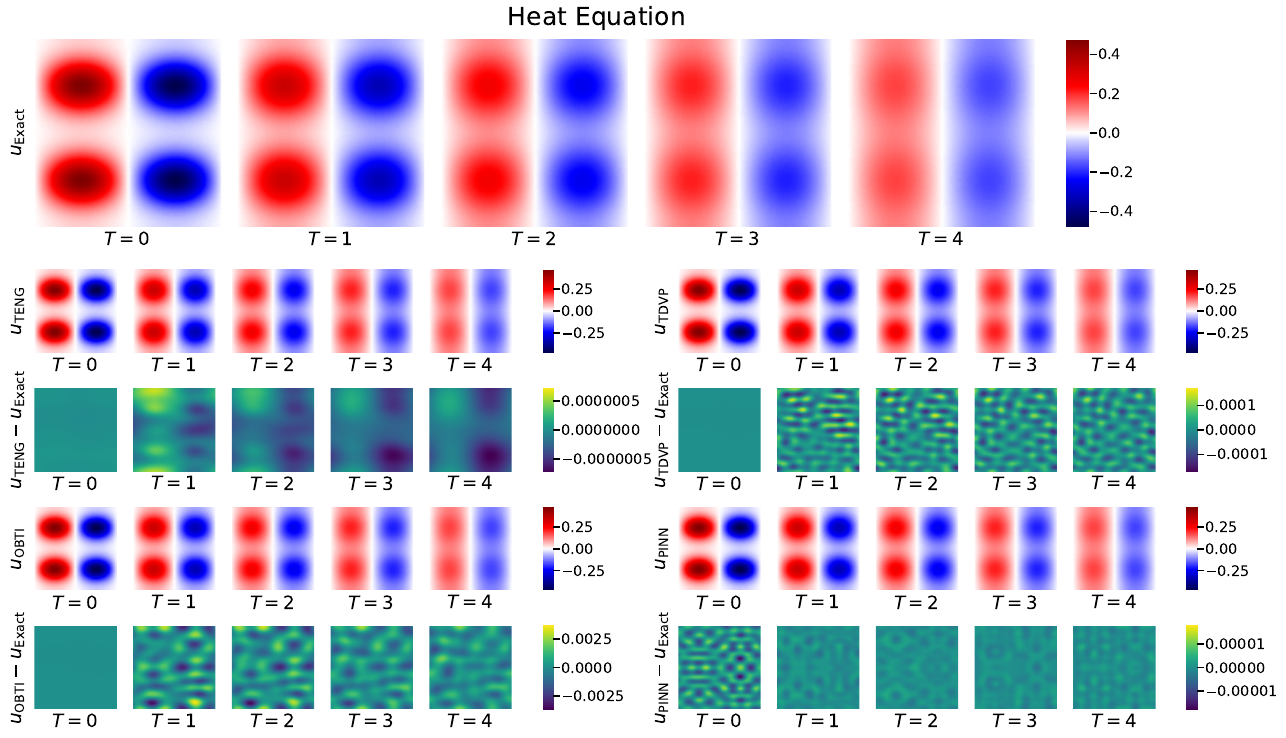}
    \caption{Exact, TENG, TDVP, OBTI, and PINN solutions and their differences from the reference solution for the two-dimensional heat equation. The reference solution is generated using the analytical solution, the TENG solution shown here uses the TENG-Heun method, the OBTI shown here uses the OBTI-LBFGS method, and the PINN shown here uses the PINN-BFGS method. The error of our TENG method is orders of magnitude smaller than other methods.}
    \label{fig:heat_color_all}
\end{figure}

\begin{figure}[ht!]
    \centering
    \includegraphics[width=0.9\linewidth]{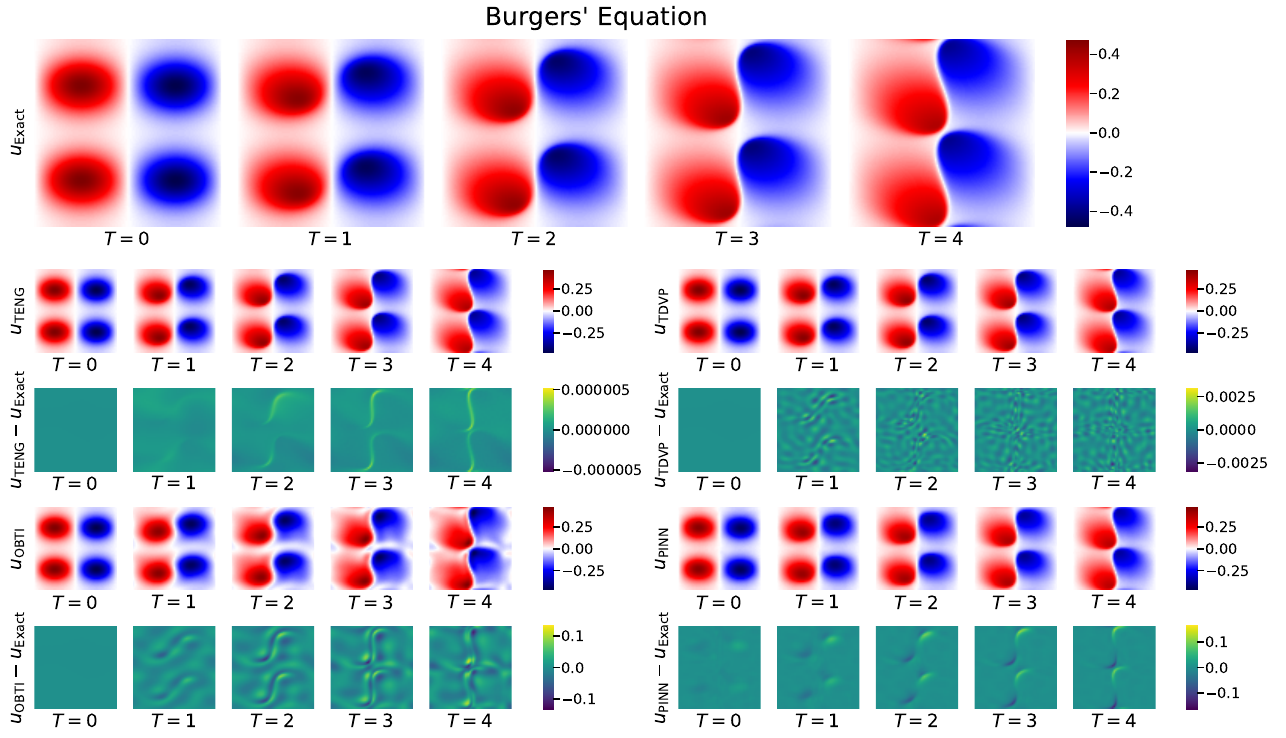}
    \caption{Exact, TENG, TDVP, OBTI and PINN solutions and their differences from the reference solution for Burgers' equation. The reference solution is generated using the spectral method, the TENG solution shown here uses the TENG-Heun method, the OBTI shown here uses the OBTI-LBFGS method, and the PINN shown here uses the PINN-BFGS method. The error of our TENG method is orders of magnitude smaller than other methods.}
    \label{fig:burgers_color_all}
\end{figure}

\begin{figure}[ht!]
    \centering
    \includegraphics[width=0.9\linewidth]{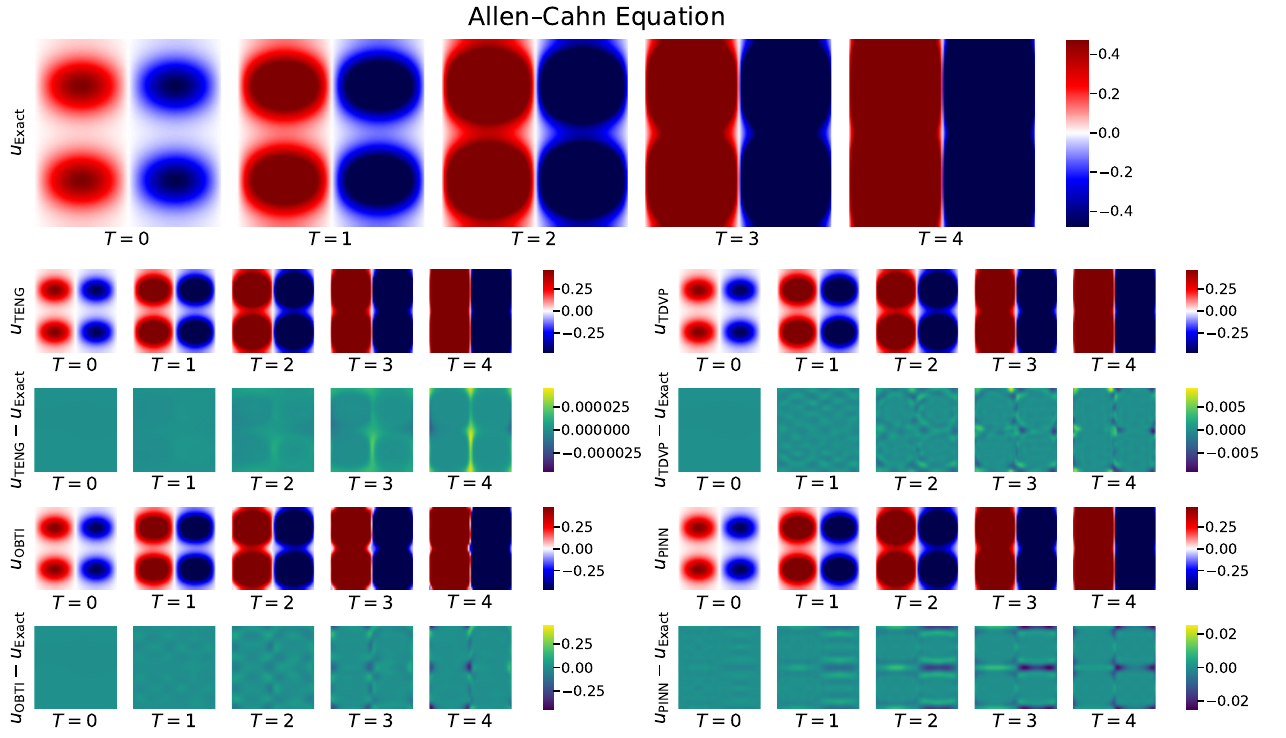}
    \caption{Exact, TENG, TDVP, OBTI and PINN solutions and their differences from the reference solution for Allen--Cahn equation. The reference solution is generated using the spectral method, the TENG solution shown here uses the TENG-Heun method, the OBTI shown here uses the OBTI-LBFGS method, and the PINN shown here uses the PINN-BFGS method. The error of our TENG method is orders of magnitude smaller than other methods.}
    \label{fig:allen_cahn_color_all}
\end{figure}

In Fig.~\ref{fig:pinn_training}, we plot the global relative $L^2$-error of PINN during training. We show both PINN-ENGD and PINN-BFGS for the heat equation, and PINN-BFGS for Allen--Cahn equation and Burgers' equation. We observe that while PINN-ENGD converges very quickly on the heat equation, PINN-BFGS eventually surpases PINN-ENGD. In addition, Allen--Cahn equation and Burgers' equation appear to be significantly more challenging for PINN, where it finds difficulty optimizing the error to below $\mathcal{O}(10^{-2})$.

\begin{figure}[ht!]
    \centering
    \includegraphics[width=\linewidth]{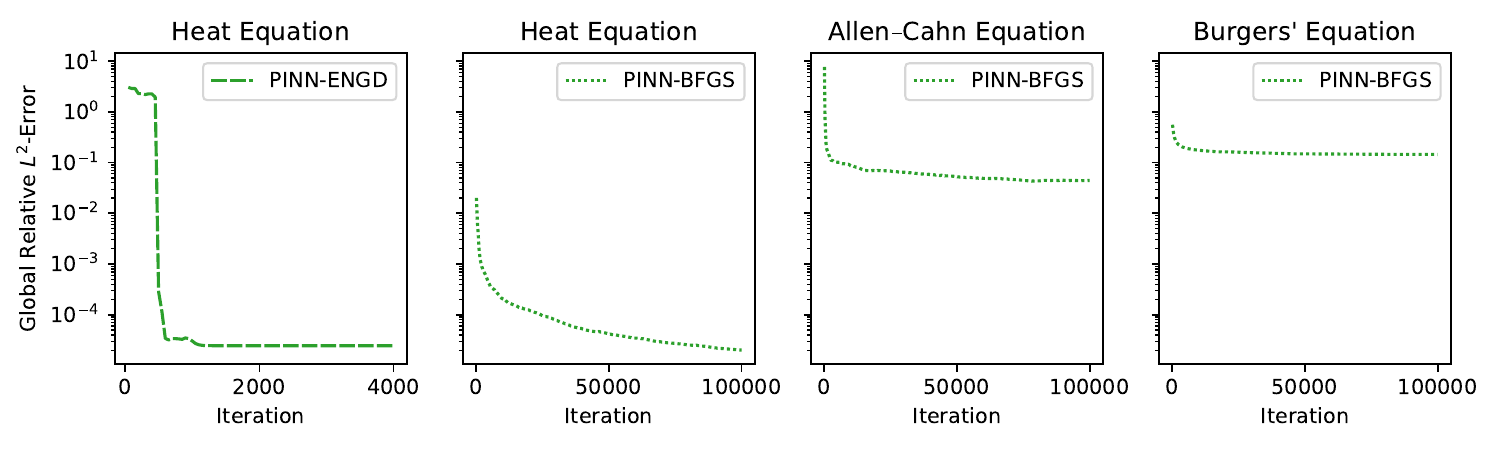}
    \caption{Global relative $L^2$-error for PINN as a function of training iterations for the two-dimensional heat equation, Allen--Cahn equation, and Burgers' equation.}
    \label{fig:pinn_training}
\end{figure}

In Fig.~\ref{fig:compare_integrator_app}, we further explore the advantage of higher-order integration schemes. In particular, we plot the relative $L^2$-error as a function of time for TENG-Euler, TENG-Heun, and TENG-RK4. For the heat equation, TENG-RK4 fails to significantly surpass TENG-Heun, which could be attributed to the error accumulation under small $\Delta t$ in this case. For the other two equations, we explore larger $\Delta t$ and find that TENG-RK4 is able to achieve small errors, while TENG-Heun's performance starts to deteriorate.

\begin{figure}[ht!]
    \centering
    \includegraphics[width=\linewidth]{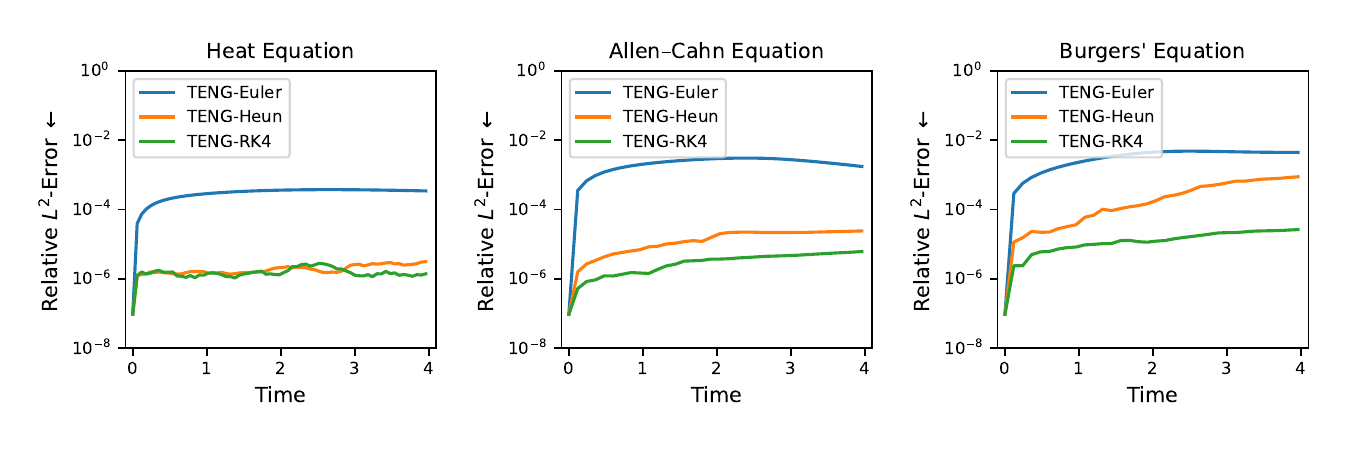}
    \caption{Relative $L^2$-error as a function of time for the two-dimensional heat equation ($\Delta t = 0.005$), Allen--Cahn equation ($\Delta t = 0.01$), and Burgers' equation ($\Delta t = 0.01$) for various integration methods.}
    \label{fig:compare_integrator_app}
\end{figure}

In Table.~\ref{table:runtime}, we report the runtime for the runs in Fig.~\ref{fig:compare_method_all}. Our TENG methods significantly improve the simulation accuracy with a similar runtime to other algorithms. We note that TENG-Heun is roughly twice as costly as TENG-Euler due to the two-stage training process in each time step. In addition, all sequential-in-time methods use significantly more time on Burgers' equation, due to the reduced time step $\Delta t$. While the result of PINN could benefit from a longer training process for the Burgers' equation, we believe it is unlikely as shown in the training dynamics in Fig.~\ref{fig:pinn_training}. In Fig.~\ref{fig:run_time_error}, we plot the global relative $L^2$-error as a function of runtime, with various choices of hyperparameters listed in Appendix~\ref{app:hyperparam}. The figure shows that TENG achieves significantly lower error compared to other methods, even for low runtimes. (The five points with the highest errors for TENG all use the Euler integration scheme, where the dominant error is the Euler discretization error.) We note that all experiments are performed on a single NVIDIA V100 GPU with 32GB memory. In all cases, the 32GB memory is sufficient for our benchmarks.

\begin{table}[ht!]
\centering
\setlength{\tabcolsep}{4pt} 
\begin{tabular}{l | c  c  c } 
\hline\hline& \\[-2.4ex]
\multirow{2}{*}{Method} & \multicolumn{3}{c}{Runtime (Hours)} \\ \cline{2-4} & \\[-2.4ex]
& Heat & Allen--Cahn & Burgers'\\ 
\hline & \\[-2.2ex]
TENG-Euler (Ours)      & 2.5              & 2.5              & 12.7 \\ 
TENG-Heun (Ours)       & 4.1              & 4.2              & 20.9\\
TDVP-RK4               & 4.6              & 4.4              & 21.1  \\
OBTI-Adam              & 3.0              & 3.2              & 19.6  \\
OBTI-LBFGS             & 4.4              & 4.1              & 22.1  \\
PINN-ENGD              & 1.1              & --                                & --  \\
PINN-BFGS              & 2.0              & 2.9              & 3.6  \\
\hline
\end{tabular}
\caption{Runtime for various algorithms for the two-dimensional heat equation, Allen--Cahn equation, and Burgers' equation.}
\label{table:runtime}
\end{table}

\begin{figure}[ht!]
    \centering
    \includegraphics[width=0.4\linewidth]{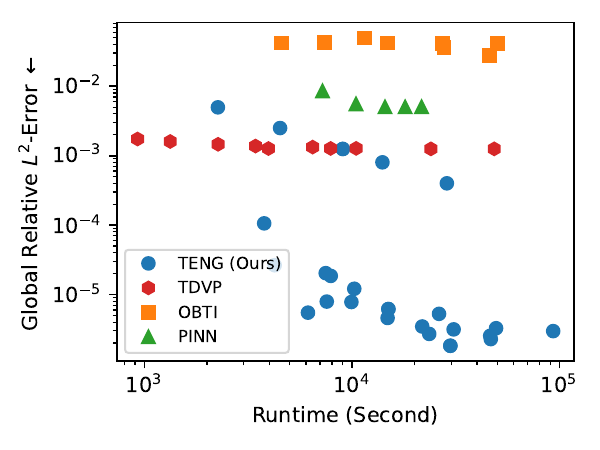}
    \caption{Global relative $L^2$-error as a function of runtime for various algorithms under various hyperparameters.}
    \label{fig:run_time_error}
\end{figure}

\section{Additional Initial Conditions and Benchmarks} \label{app:init_cond}
As mentioned in the main paper, for the three-dimensional heat equation, we consider a initial condition in the form
\begin{equation}
\begin{aligned}
    u_0(x_1, x_2, x_3) &= A_{000} 
      + \sum_{k_1=1}^{2}\sum_{k_2=1}^{2}\sum_{k_3=1}^{2} \left(A_{k_1k_2k_3} \prod_{i=1}^3\cos\left(k_i x_i\right)
      + B_{k_1k_2k_3} \prod_{i=1}^3\sin\left(k_i x_i\right)\right).
\end{aligned}
\end{equation}
Here, we choose the following coefficients: $A_{000} = 0.043$ with the rest of $A_{k_1k_2k_3}$'s and $A_{k_1k_2k_3}$'s shown in Table~\ref{tab:coefs}.

\begin{table}[h!]
\centering
\begin{tabular}{|l|cc|cc|}
\hline
 \multirow{2}{*}{$A_{k_1 k_2 k_3}$} & \multicolumn{2}{c}{$k_1=1$} & \multicolumn{2}{c|}{$k_1=2$} \\
\cline{2-5}
 & $k_2=1$ & $k_2=2$ & $k_2=1$ & $k_2=2$ \\
\hline
$k_3=1$ & 0.047 & -0.021 & 0.034 & -0.02 \\
$k_3=2$ & -0.021 & -0.041 & 0.024 & 0 \\
\hline
\end{tabular}\\
\medskip
\begin{tabular}{|l|cc|cc|}
\hline
 \multirow{2}{*}{$B_{k_1 k_2 k_3}$} & \multicolumn{2}{c}{$k_1=1$} & \multicolumn{2}{c|}{$k_1=2$} \\
\cline{2-5}
 & $k_2=1$ & $k_2=2$ & $k_2=1$ & $k_2=2$ \\
\hline
$k_3=1$ & -0.075 & -0.056 & -0.027 & -0.008 \\
$k_3=2$ & 0.074 & -0.007 & 0.032 & 0 \\
\hline
\end{tabular}
\caption{$A_{k_1 k_2 k_3}$'s and $B_{k_1 k_2 k_3}$'s for the initial condition of three-dimensional heat equation.}
\label{tab:coefs}
\end{table}

In addition, we consider an example of the heat equation defined on a two-dimensional disk with Dirichlet boundary condition. Here, the boundary condition is enforced via an additional loss term in Eq.~\eqref{eq:obti_loss}, and the initial condition is shown below.
\begin{equation} \label{eq:initial_condition}
\begin{aligned}
    u_0(r, \theta) = \frac{1}{4}\bigg(& Z_{01}(r, \theta) - \frac{1}{4}Z_{02}(r, \theta) + \frac{1}{16}Z_{03}(r, \theta) - \frac{1}{64}Z_{04}(r, \theta) \\
    +& Z_{11}(r, \theta) - \frac{1}{2}Z_{12}(r, \theta) + \frac{1}{4}Z_{13}(r, \theta) - \frac{1}{8}Z_{14}(r, \theta) + Z_{21}(r, \theta) + Z_{31}(r, \theta) + Z_{41}(r, \theta) \bigg),
\end{aligned}
\end{equation}
where $r$ and $\theta$ is the polar coordinate variables and $Z_{mn}$ represent the disk harmonics defined as
\begin{equation}
    Z_{mn}(r, \theta) = J_{m}(\lambda_{nm} r) \cos(m \theta)
\end{equation}
with $J_m$ the $m$th Bessel function and $\lambda_{nm}$ the $n$th zero of the $m$th Bessel function. We note that while the analytical solution is solved in the polar coordinates, all neural network based methods solve the equation and benchmark in the original Cartesian coordinates.

For Burgers' equation, we, in addition, consider benchmarks that include a case with smaller $\nu=3/1000$ with the original domain, boundary, and initial conditions, and a case with $\nu=1/100$ but with nonequal domain $\mathcal{X} = [0, 2) \times [0, 2\pi)$ with periodic boundary condition, and $\mathcal{T} = [0, 4]$, and the following initial condition.
\begin{equation} \label{eq:initial_condition}
\begin{aligned}
    u_0(x_1, x_2) = \frac{1}{50}\exp\left(\cos\left(\pi x_1 - 2\right) + \sin\left(x_2 - 1\right) \right) ^ 2.
\end{aligned}
\end{equation}

\begin{figure}[ht!]
    \centering
    \includegraphics[width=\linewidth]{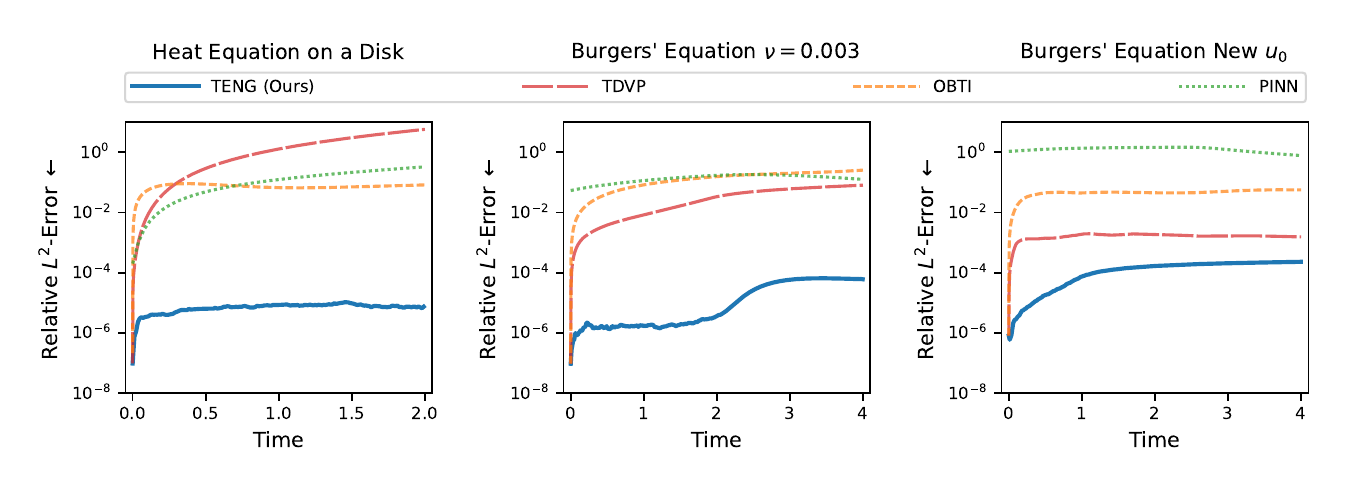}
    \caption{Relative $L^2$-error as a function of time for additional benchmarks. For all sequential-in-time methods, we choose time step size $\Delta t=0.005$ for the heat equation and $\Delta t = 0.001$ for Burgers' equation.}
    \label{fig:benchmark_add}
\end{figure}

In Fig.~\ref{fig:benchmark_add}, we show the additional benchmarks for the aforementioned examples. Here, TENG refers to the TENG-Heun method, OBTI refers to the OBTI-LBFGS method, and PINN refers to the PINN-BFGS method. We note that for the heat equation on a disk with Dirichlet boundary condition, an additional boundary term is included in the loss function defined in Eq.~\eqref{eq:obti_loss} for TENG and OBTI method. (PINN can also incorporate this boundary term analogously.) However, it is unclear how to enforce the boundary condition in TDVP without redesigning the neural network architecture; therefore, we choose to not enforce the boundary condition for TDVP, which could be the reason why TDVP performs particularly badly on the heat equation on a disk.

\section{Details on Obtaining Reference Solutions} \label{app:exact}
{\bf Heat equation.} As mentioned in the main paper, the heat equation permits an analytical solution in terms of Fourier series. For example, we show the two-dimensional case below.
\begin{equation}
    u(x_1,x_2,t) = \sum_{k_1,k_2} \exp \left( -\nu \left(k_1^2 + k_2^2\right) t \right) \tilde{u}_0(k_1, k_2) \exp\left(i k_1 x_1 + i k_2 x_2\right),
\end{equation}
where we omit the terms $2\pi / P$ because in our case $P=2\pi$.
For the two-dimensional case, evaluating the analytical solution is not practical since it is difficult to express our initial condition in Fourier series analytically, 
\begin{equation}
    \tilde{u}_0(k_1, k_2) = \sum_{x_1, x_2} u_0(x_1, x_2)\exp\left(-i k_1 x_1 -i k_2 x_2\right)
\end{equation}
not to mention calculating an infinite sum of frequencies. Therefore, we choose to evaluate the initial condition on a $2048\times 2048=4194304$ grid. Then, we use the discrete Fourier transform to calculate the initial condition in the Fourier space, before truncating the maximum frequency $48$. (The summation contains around $(2 \cdot 48)^2\approx9000$ terms in total). For the three-dimensional case and the case where the domain is a disk, since the initial condition is already defined in terms of sinusoidal functions (or Bessel functions), the solution is analytically calculated.

{\bf Allen--Cahn equation.} Different from the heat equation, Allen--Cahn equation generally does not permit analytical solutions. Therefore, we solve it using the spectral method and consider the solution as a \textit{proxy} for the exact solution as the reference. Here, the basis functions of the spectral method are chosen to be the same Fourier plane waves, so the solution in real space can be written as
\begin{equation}
    u(x_1,x_2,t) = \sum_{k_1,k_2} \tilde{u}(k_1, k_2, t) \exp\left(i k_1 x_1 + i k_2 x_2\right).
\end{equation}
When switching from real space to Fourier space, we have
\begin{equation}
    \frac{\partial u}{\partial x_j} \rightarrow i k_j \tilde{u} \quad \text{and} \quad uv \rightarrow \tilde{u} \circ \tilde{v},
\end{equation}
where $\circ$ means convolution. Therefore, the PDE can be rewritten in the Fourier space as
\begin{equation} \label{eq:allen_cahn_freq}
    \frac{\partial\tilde{u}}{\partial t} = -\nu(k_1^2 + k_2^2) \tilde{u} + \tilde{u} - \tilde{u}\circ\tilde{u}\circ\tilde{u}.
\end{equation}
Here, we choose a maximum frequency cut-off of 128. (Notice that the maximum number of frequencies encountered is $(3 \cdot 2 \cdot 128)^2\approx600000$ when calculating the double convolution.) The initial condition is calculated analogous to the case of the heat equation, via a discrete Fourier transform on the $2048\times 2048=4194304$ grid. Then, Eq.~\eqref{eq:allen_cahn_freq} is solved using the fourth-order Runge--Kutta integration scheme with a time step $\Delta t=\num{2e-4}$.

{\bf Burgers' equation.} Analogous to Allen--Cahn equation, Burger's equation does not have a general analytical solution either, except in the case of $\nu=0$. Therefore, we use the same spectral method used to solve Allen--Cahn equation. Notice that the term $u \partial u / \partial x_j = \partial u^2 /\partial x_j$. Therefore, Burgers' equation in the Fourier space reads
\begin{equation} \label{eq:burgers_freq}
    \frac{\partial\tilde{u}}{\partial t} = -\nu(k_1^2 + k_2^2) \tilde{u} - \frac{i}{2} (k_1 + k_2) \tilde{u} \circ \tilde{u}.
\end{equation}
Here, we choose a maximum frequency cut-off of 192 (with a maximum of around $(2 \cdot 2 \cdot 192)^2\approx 600000$ terms when calculating the convolution.) The initial condition is calculated in the same way as the heat and Allen--Cahn equation, and Eq.~\eqref{eq:burgers_freq} is solved using the fourth-order Runge--Kutta integration scheme with a time step $\Delta t=\num{1e-4}$.

{\bf Accuracy of the solutions.} In each case, we carefully verify that the number of grid points, the maximum frequency, and the $\Delta t$ are sufficient to obtain a solution that is accurate to near numerical precision, by varying them over multiple values and observing that the solution converges. We note that the case for Burgers' equation with $\nu=0.003$ is challenging for the spectral method and the solution may not converge yet, which means the errors we report could be larger than the actual values.

\section{Details of Neural Network Architectures and Optimization} \label{app:hyperparam}

All the algorithms used in this work are implemented in JAX and use double precision. Our code is posted on GitHub at \url{https://github.com/pde-sim/teng}.

{\bf Neural network architectures.} 
We choose the same architecture for all sequential-in-time methods, which allows a fair comparison. Our neural network architecture is loosely based on Ref.~\cite{berman2023randomized} which consists of multiple feedforward layers with tanh activation function as
\begin{equation}
    \uhat(x) = W_{n_l} \mathrm{tanh}\left(\cdots\mathrm{tanh}\left(W_1 \mathrm{periodic\_embed}(x) + b_1\right)\cdots\right) + b_{n_l},
\end{equation}
where the periodic embedding function is defined as
\begin{equation}
    \mathrm{periodic\_embed}(x) = \mathrm{concatenate}\left(\left[\textstyle\sum_j a_j\cos\left(x_1 + \phi_j\right) + c_j, \textstyle\sum_j a_j\cos\left(x_2 + \phi_j\right) + c_j \right]\right)
\end{equation}
to explicitly enforce the periodic boundary condition in the neural network. Here, all $W$, $b$, $a$, and $c$ are trainable parameters. Here, we choose $n_l=7$ layers and $d_h=40$ hidden dimensions (periodic embedding vector with size 20 for each $x_j$).

For PINN, we adopt the same architecture from Ref.~\cite{muller2023achieving} with the addition of periodic embedding. In addition, we increase the hidden dimension to 64 compared to  Ref.~\cite{muller2023achieving} for better expressivity. 

In the case of the heat equation on a 2D disk, we simply remove the periodic embedding layer.

{\bf Optimization methods.}
For TENG, we randomly sub-sample trainable parameters when solving the least square problems. This can be viewed as a regularization method when the original least square problem is ill-conditioned and can significantly reduce the computational cost. During each time step, we randomly sub-sample 1536 parameters in the first iteration and sub-sample 1024 parameters in the subsequent iterations. In TENG-Euler, the neural network is optimized for 7 iterations in each time step; in TENG-Heun, the neural network is optimized for 7 iterations to obtain $\theta_\mathrm{temp}$, followed by 5 iterations for $\theta_{t+\Delta t}$. We reduce the number of iterations in the second stage because $\theta_\mathrm{temp}$ already gives a good initialization for $\theta_{t+\Delta t}$.

For TDVP, we use the sparse update method proposed by Ref.~\cite{berman2023randomized}, which is also a random sub-sample of parameters for each TDVP step, and has been shown to significantly improve the result compared to a full update of a smaller neural network. Here, we randomly sub-sample 2560 parameters at each time step so that the computational cost of TDVP at each time step roughly matches that of TENG (over the training iterations within each time step).  

For OBTI, we compare our method with both the Adam optimizer and the L-BFGS optimizer. Within each time step, the neural network is optimized for 300 iterations when using the Adam optimizer, and 150 iterations when using the L-BFGS optimizer. The Adam optimizer uses an initial learning rate of $\num{1e-5}$ and an exponential scheduler that decays the learning rate by $1/2$ by the end of the 300 iterations.

For all sequential-in-time methods, we need to train the initial parameters to match the initial conditions. Here we use the same initial parameters for a fair comparison. The initial parameters are trained by first minimizing the loss function
\begin{equation}
    L(\hat{u}_\theta, u_0) = \norm{\hat{u}_\theta - u_0}_{L^2(\mathcal{X})}^2 + \norm{\frac{\partial\hat{u}_\theta}{\partial x_1} - \frac{\partial u_0}{\partial x_1}}_{L^2(\mathcal{X})}^2 + \norm{\frac{\partial\hat{u}_\theta}{\partial x_2} - \frac{\partial u_0}{\partial x_2}}_{L^2(\mathcal{X})}^2
\end{equation}
using natural gradient descent, where we use the least square formulation as mentioned in the main paper and (approximately) solve the least square problem using CGLS method, until the loss value decays below $\num{1e-7}$. Then, we switch the loss function to
\begin{equation}
    L(\hat{u}_\theta, u_0) = \norm{\hat{u}_\theta - u_0}_{L^2(\mathcal{X})}^2
\end{equation}
and use the random sub-sample version of the natural gradient descent, with 1536 parameters updated for each iteration until the loss value decays to near machine precision ($\num{1e-14}$). The $L^2$-norm in both stages are integrated on a 2D grid of 1024 points in each dimension (around 1000000 points in total).

For PINN, both the initial condition and the time evolution are optimized simultaneously; therefore, it does not use the initial parameters mentioned above. In addition, all the time steps of PINN are optimized simultaneously, instead of step by step. For the optimization, we test the BFGS optimizer, and the recently proposed ENGD optimizer~\cite{berman2023randomized}. We note that the ENGD optimizer requires custom implementation for individual PDEs. Since Ref.~\cite{berman2023randomized} did not provide the implementation for Allen--Cahn equation and Burgers' equation, we omit the benchmark of ENGD optimizer for the two equations. We train the neural network for 100000 iterations when using the BFGS optimizer, and 4000 iterations when using the ENGD optimizer. 

For Fig~\ref{fig:run_time_error}, the results include various hyperperameters. For all sequential-in-time methods, we include different time step sizes $\Delta t = 0.0016$, $0.0032$, $0.005$, $0.01$ and $0.02$. For TENG, we include TENG-Euler, TENG-Heun, and TENG-RK4 with different numbers of iterations (within each time step) ranging from $2$ to $20$ and different numbers of randomly subselected parameters for solving least squares (within each iteration) ranging from $384$ to $2048$; for TDVP, we include different numbers of randomly subselected parameters for solving least square projections ranging from $384$ to $2560$ (where we reach the memory limit of V100 GPU); for OBTI, we include both OBTI-Adam and OBTI-LBFGS with different numbers of iterations (within each time step) ranging from $150$ to $300$; and for PINN, we use the BFGS optimizer, and include results of different number of iterations (globally) and different neural network sizes.


\end{document}